\documentclass[10pt]{article}
\usepackage{pdfpages,fullpage,times}

\usepackage{graphicx,subfigure,color,tikz}
\usepackage{caption,booktabs}
\usepackage{algorithm,algorithmic,hyperref}
\usepackage{caption}

\usepackage{amsthm,amssymb,amsmath}
\usepackage{xspace,enumitem}

\newcommand{\eg}{\emph{e.g.}}
\newcommand{\ie}{\emph{i.e.}}

\newcommand{\EE}{\mathbb{E}} 
\newcommand{\RR}{\mathbb{R}} 

\newcommand{\Hcal}{\mathcal{H}} 
\newcommand{\Otil}{\tilde{O}}
\newcommand{\Ytil}{\tilde{Y}}

\newcommand*{\argmin}{\mathop{\mathrm{argmin}}}

\newcommand{\inner}[2]{\left\langle #1,#2 \right\rangle}
\newcommand{\rbr}[1]{\left(#1\right)}
\newcommand{\sbr}[1]{\left[#1\right]}
\newcommand{\cbr}[1]{\left\{#1\right\}}
\newcommand{\nbr}[1]{\left\|#1\right\|}

\newcommand{\tr}{\mathop{\mathrm{tr}}}
\newcommand{\diag}{\mathop{\mathrm{diag}}}

\newtheorem{theorem}{Theorem}
\newtheorem{lemma}{Lemma}
\newtheorem{claim}{Claim}
\newtheorem{definition}{Definition}
\newtheorem*{rep@theorem}{\rep@title}
\newenvironment{oneshot}[1]{\def\rep@title{#1} \begin{rep@theorem}}{\end{rep@theorem}}
\newenvironment{proofsketch}{\par {\it Proof Sketch.}}{\endproof}

\newcommand{\CP}{\textrm{Master}}
\newcommand{\CS}{$\textsc{CountSketch}$\xspace}
\newcommand{\TS}{$\textsc{TensorSketch}$\xspace}
\newcommand{\nnz}{\mathrm{nnz}}
\newcommand{\nsparse}{\rho}

\title{Communication Efficient Distributed Kernel Principal Component Analysis}

\author{
Maria-Florina Balcan\thanks{School of Computer Science, Carnegie Mellon University. Email: \texttt{ninamf@cs.cmu.edu}}, ~
Yingyu Liang\thanks{Department of Computer Science, Princeton University. Email: \texttt{yingyul@cs.princeton.edu}},~
Le Song\thanks{College of Computing, Georgia Institute of Technology. Email:\texttt{lsong@cc.gatech.edu}},~
David Woodruff\thanks{Almaden Research Center, IBM Research. Email: \texttt{dpwoodru@us.ibm.com}},~
Bo Xie\thanks{College of Computing, Georgia Institute of Technology. Email:\texttt{bo.xie@gatech.edu}}
}
\date{}

\begin{document}

\maketitle
\begin{abstract}
Kernel Principal Component Analysis (KPCA) is a key machine learning algorithm for extracting nonlinear features from data. In the presence of a large volume of high dimensional data collected in a distributed fashion, it becomes very costly to communicate all of this data to a single data center and then perform kernel PCA. Can we perform kernel PCA on the entire dataset in a distributed and communication efficient fashion while maintaining provable and strong guarantees in  solution quality?    

In this paper, we give an affirmative answer to the question by developing a communication efficient algorithm to perform kernel PCA in the distributed setting. The algorithm is a clever combination of subspace embedding and adaptive sampling techniques, and we show that the algorithm can take as input an arbitrary configuration of distributed datasets, and compute a set of global kernel principal components with relative error guarantees independent of the dimension of the feature space or the total number of data points. In particular, computing $k$ principal components with relative error $\epsilon$ over $s$ workers has communication cost $\tilde{O}(s \nsparse k/\epsilon+s k^2/\epsilon^3)$ words, where $\nsparse$ is the average number of nonzero entries in each data point. Furthermore, we experimented the algorithm with large-scale real world datasets and showed that the algorithm produces a high quality kernel PCA solution while using significantly less communication than alternative approaches. 
\end{abstract}

\section{Introduction} \label{sec:intro}


Kernel Principal Component Analysis (KPCA) is a key machine learning algorithm for extracting nonlinear features from complex datasets, such as image, text, healthcare and biological data~\cite{SchSmoMul97,SchSmo02,SchTsuVer04}. The original kernel PCA algorithm is designed for a batch setting, where all data points need to fit into a single machine. However, nowadays large volumes of data are being collected increasingly in a distributed fashion, which poses new challenges for running kernel PCA. For instance, a large network of distributed sensors can collect temperature readings from geographically distant locations; a system of distributed data centers in an Internet company can process user queries from different countries; a fraud detection system in a bank needs to perform credit checks on people opening accounts from different branches; and a network of electronic healthcare systems can store patient records from different hospitals. It is very costly in terms of network bandwidth and transmission delays to communicate all of the data collected in a distributed fashion to a single data center, and then run kernel PCA on the central node. In other words, communication now becomes the bottleneck to the nonlinear feature extraction pipeline. How can we leverage the aggregated computing power in a large distributed system? Can we perform kernel PCA on the entire dataset in a distributed and communication efficient fashion while maintaining provable and strong guarantees in solution quality?

While recent work shows how to do linear PCA in a communication efficient and distributed fashion~\cite{BouSviWoo15}, the kernel setting is significantly more challenging. The main problem with previous work is that it achieves communication proportional to the dimension of the data points, which if implemented straightforwardly in the kernel setting would give communication proportional to the dimension of the feature space which can be very large or even infinite. Kernel PCA uses the kernel trick to avoid going to the potentially infinite dimensional kernel feature space explicitly, so intermediate results are often represented by a function (\eg, a weighted combination) of the feature mapping of some data points. Communicating such intermediate results requires communicating all the data points they depend on. To lower the communication, the intermediate results should only depend on a small number of data points. A distributed algorithm then needs to be carefully designed to meet this constraint. 

In this paper, we propose a communication efficient algorithm for distributed KPCA in a master-worker setting where
the dataset is arbitrarily partitioned and each portion sits in one worker, and the workers can communicate only through the master. Our key idea  is to design a communication efficient way of generating a small representative subset of the data, and then performing kernel PCA based on this subset. We show that the algorithm can compute a rank-$k$ subspace in the kernel feature space using just a representative subset of size $O(k/\epsilon)$ built in a distributed fashion. For polynomial kernels, it achieves a $(1+\epsilon)$ relative-error approximation to the best rank-$k$ subspace, and for shift-invariant kernels (such as the Gaussian kernel), it achieves $(1+\epsilon)$-approximation with an additive error term that can be made arbitrarily small. In both cases, the total communication for a system of $s$ workers is $\tilde{O}(s \nsparse k/\epsilon+s k^2/\epsilon^3)$ words, where $\nsparse$ is the average number of nonzero entries in each data point, and is always bounded by the dimension of the data $d$ and independent of the dimension of the kernel feature space. This for constant $\epsilon$ nearly matches the lower bound $\Omega(sdk)$ for linear PCA~\cite{BouSviWoo15}.
As far as we know, this is the first algorithm that can achieve provable approximation with such communication bounds. 

As a subroutine of our algorithm, we have also developed an algorithm for the distributed Column Subset Selection (CSS) problem, which can select a set of $O(k/\epsilon)$ points whose span contains $(1+\epsilon)$-approximation, with communication $O(s\nsparse k /\epsilon + s k^2)$. This is the first algorithm that addresses the problem for kernels, and it nearly matches the communication lower bound $\Omega(s\nsparse k /\epsilon)$ for this problem in the linear case~\cite{BouWoo15}. The column subset selection problem has various applications in big data scenarios, so this result could be of independent interest.

Furthermore, our algorithm also leads to some other distributed kernel algorithms: the data can then be projected onto the subspace found and processed by downstream applications. For example, 
an immediate application is for distributed spectral clustering, that first computes KPCA to rank-$k/\epsilon$ and then does $k$-means on the data projected on the subspace found by KPCA (\eg, \cite{DhiGuaKul04}). This can be done by combining our algorithm with any efficient distributed $k$-means algorithms (\eg, \cite{BalKanLiaWoo14}).

We evaluate our algorithm on datasets with millions of data points and hundreds of thousands of dimensions where non-distributed algorithms such as batch KPCA are impractical to run. Furthermore, comparing to other distributed algorithms, our algorithm requires less communication and fewer representation data points to achieve the same approximation error.

{\bf Outline} Section~\ref{sec:relatedwork} reviews related work, and Section~\ref{sec:preliminary} reviews some preliminaries. Section~\ref{sec:intuition} provides an overview. Section~\ref{sec:disKPCA} presents our distributed kernel PCA algorithm, which consists of the key building blocks: kernel subspace embedding, computing (generalized) leverage scores, sampling representative points, and finally computing the solution in the span of the selected data points. Section~\ref{sec:exp} provides empirical results. Due to space limitation, 
the details of some proofs are deferred to the full version~\cite{LiaXieWooSonBal15}.

\section{Related Work}  \label{sec:relatedwork}

There has been a surge of recent work on distributed machine learning, \eg,~\cite{BalBluFinMan12,ZhaWaiDuc12,KanVemWoo14,BalKanLiaWoo14}. 
In this setting, the data sets are typically large, and small error rate is required. This is because if only a coarse error is needed then there is no need to use large-scale data sets; a small subset of the data will be sufficient. Furthermore, one prefers relative error rates instead of additive error rates, since the latter is worse and harder to interpret without knowing the optimum. Our algorithm can achieve small relative error with limited communication. 

Since there exist communication efficient distributed linear PCA algorithms~\cite{BalKanLiaWoo14,KanVemWoo14}, it is tempting to adopt the random feature approach for distributed kernel PCA: first construct $m$ random features and then solve PCA in the primal form, \ie, apply distributed linear PCA on the random features.
However, the communication of this method is too high. One needs $m=\tilde{O}(d/\epsilon^2)$ random features to preserve the kernel values up to additive error $\epsilon$, leading to a communication of $O(skm/\epsilon) = O(skd/\epsilon^3)$.
 Another drawback of using random features is that it only produces a solution in the space spanned by the random features, but not a solution in the feature space of the kernel. 

The Nystr\"om method is another popular tool for large-scale kernel methods: sample a subset of data points uniformly at random, and use them to construct an approximation of the original kernel matrix. However, it also suffers from high communication cost, since one needs $O(1/\epsilon^4)$ sampled points to achieve additive $\epsilon$ error in the Frobenius norm of the kernel matrix~\cite{KumMohTal12}. A closely related method is incomplete Cholesky decomposition~\cite{BacJor05}, where a few pivots are greedily chosen to approximate the kernel matrix. It is unclear how to design a communication efficient distributed version since it requires as many rounds of communication as the number of pivots, which is costly.


Leverage score sampling is a related technique for low-rank approximation~\cite{Woodruff14}. A prior work of Boutsidis et al. \cite{BouSviWoo15} gives the first distributed protocol for column subset selection. \cite{BouWooZho15} gives a distributed PCA algorithm with optimal communication cost, but only for linear PCA.
In comparison, our work is the first communication efficient distributed algorithm for low rank approximation in the kernel space. 

\section{Backgrounds} \label{sec:preliminary}

For any vector $v$, let $\nbr{v}$ denote its Euclidean norm.
For any matrix $M \in \RR^{d\times n}$, let $M_{i:}$ denote its $i$-th row and $M_{:j}$ its $j$-th column. Let $\nbr{M}_F$ denote its Frobenius norm, and $\nbr{M}_2$ denote its spectral norm. 
Let its rank be $r \leq \min\cbr{n,d}$, and denote its SVD as $M = U \Sigma V^\top$ where $U \in \RR^{d\times r}, \Sigma \in \RR^{r\times r}$, and $V \in \RR^{n \times r}$. 
Let $[M]_k$ denote its best rank-$k$ approximation.
Finally, denote its number of non-zero entries as $\nnz(M)$.

In the distributed setting, there are $s$ workers that are connected to a master processor.  Worker $i$ has a local data set $A^i \in \RR^{d \times n_i}$, and the global data set $A \in \RR^{d \times n}$ is the concatenation of the local data ($n = \sum_{i=1}^s n_i$). 

{\bf Kernels and Random Features.} For a kernel $\kappa(x,x')$, let $\Hcal$ denote its feature space, i.e., there exists a feature mapping $\phi(\cdot) \in \Hcal$ such that $\kappa(x,x') = \inner{\phi(x)}{\phi(x')}_\Hcal$. Let $\phi(A) \in \Hcal^n$ denote the matrix obtained by applying $\phi$ on each column of $A$ and concatenating the results. Throughout the paper, we regard any $M \in \Hcal^n$ as a matrix whose columns are elements in $\Hcal$ and define matrix operations accordingly. For example, for any $M \in \Hcal^{n}$ and $N \in \Hcal^{m}$, let $B = M^\top N \in \RR^{n\times m}$ where $B_{ij} = \inner{M_{:i}}{N_{:j}}_\Hcal$, and let $\nbr{M}^2_\Hcal = \tr\rbr{M^\top M}$. When there is no ambiguity, we omit the subscript $\Hcal$.

The random feature approach is a recent technique to scale up kernel methods. Many kernels can be approximated by $\frac{1}{m} \sum_{i=1}^m { \xi_{\omega_i}(x) \xi_{\omega_i}(y) }$ where $\omega_i$'s are randomly sampled. These include Gaussian RBF kernels and other shift-invariant kernels, inner product kernels, etc (\cite{RahRec07,DaiXieHe14}). 
For example, Gaussian RBF kernels, $\kappa(x,y) = \exp(-\|x-y\|^2/2\sigma^2)$, can be approximated by $\frac{1}{m}\sum_{i=1}^m {z_{\omega_i,b_i}(x) z_{\omega_i,b_i}(y)}$ where $z_{\omega,b}(x)  = \sqrt{2}\cos(\omega^\top x + b)$ and $\omega_i$ is from a Gaussian distribution with density proportional to $\exp(-\sigma^2 \nbr{\omega}^2/2)$ and $b_i$ is uniform over $[0,2\pi]$. 

In this paper, we provide guarantees for shift-invariant kernels using Fourier random features (the extension to other kernels/random features is straightforward). 
We assume the kernel satisfies some regularization conditions: it is defined over bounded compact domain in $\RR^d$, with $\kappa(0)\leq 1$ and bounded $\nabla^2 k(0)$~\cite{RahRec07}. Such conditions are standard in practice, and thus we assume them throughout the paper.

{\bf Kernel PCA.} 
An element $u \in \Hcal$ is an eigenfunction of $\phi(A)\phi(A)^\top$ with the corresponding eigenvalue $\lambda$ if $\nbr{u} = 1$ and $\phi(A)\phi(A)^\top u = \lambda u$. Given eigenfunctions $\cbr{u_i}$ of $\phi(A)\phi(A)^\top$ and eigenvectors $\cbr{v_i}$ of $\phi(A)^\top\phi(A)$, $\phi(A)$ has the singular decomposition  $U\Sigma_k V^\top + U_{\perp}\Sigma_{\perp} V_{\perp}^\top$,
where $U$, $V$ are the lists of top $k$ eigenfunctions/vectors, $\Sigma_k$ is a diagonal matrix with the corresponding singular values, $U_{\perp}$, $V_{\perp}$ are the lists of the rest of the eigenfunctions/vectors, and $\Sigma_{\perp}$ is a diagonal matrix with the rest of the singular values. 
Kernel PCA aims to identify the top $k$ subspace $U$, since the best rank-$k$ approximation $[\phi(A)]_k = U\Sigma_k V^\top = UU^\top \phi(A)$.  Typically, the goal is to find a good approximation to this subspace. Formally,
\begin{definition} \label{def:kernelLA}
A subspace $L \in \Hcal^{k}$ is a rank-$k$ $(1+\epsilon, \Delta)$-approximation for kernel PCA on $A$ if $L^\top L = I_k$ and
$$
  \nbr{\phi(A) - L L^\top \phi(A) }^2 \leq (1+\epsilon) \nbr{\phi(A) - \sbr{\phi(A)}_k }^2 + \Delta.
$$
\end{definition}
Note that kernel PCA immediately leads to solutions for some other nonlinear component analysis such as  kernel CCA, and also provides the needed subroutine for tasks like spectral clustering. 

{\bf Subspace Embeddings.}
Subspace embeddings are a useful technique that can improve the computational and space costs by embedding 
data points into lower dimension while preserving interesting properties. Subspace embeddings have been extensively studied in recent years~\cite{Sar06,Achlioptas03,ArrVem99,AilCha09,ClaWoo13}.
The recent fast sparse subspace embeddings~\cite{ClaWoo13} and its optimizations~\cite{MenMah13,NelNgu13} are particularly suitable for large-scale sparse datasets, since their running time is linear in the number of non-zero entries in the data matrix. They also preserve the sparsity of the input data. 
Formally, 
\begin{definition}\label{def:OSE}
An $\epsilon$-subspace embedding of $M \in \RR^{m \times n}$ is a matrix $S\in \RR^{t \times m}$ such that 
for any $x$, 
\[
  \nbr{SM x} = (1\pm \epsilon) \nbr{Mx}.
\]
Subspace embeddings can also be done on the right hand side, \ie, $S \in \RR^{n\times t}$ and $\nbr{x^\top MS} = (1\pm \epsilon) \nbr{x^\top M}$.
\end{definition}

$Mx$ is in the column space of $M$ and $SMx$ is its embedding, so the definition means that the norm of any vector in the column space of $M$ is approximately preserved. This then provides a way to do dimensional reduction for problems depending on inner products of vectors.
Our algorithm repeatedly makes use of subspace embeddings. In particular, the embedding we use is the concatenation of the following known sketching matrices: \CS and i.i.d.\ Gaussians (or the concatenation of \CS, fast Hadamard and i.i.d. Gaussians). The details can be found in~\cite{Woodruff14}; we only need the following fact.

\begin{lemma}\label{fac:embedding}
For $M \in \RR^{d \times n}$, there exist sketching matrices $S \in \RR^{t \times d}$ with $t = O(n/\epsilon^2)$ that are $\epsilon$-subspace embeddings. Furthermore, $SM$ can be successfully computed in time $\tilde{O}(\nnz(M))$ with probability at least $1-\delta$.
\end{lemma}

The work of \cite{AvrNguWoo14} shows that a fast computational approach, \TS, is indeed a subspace embedding for the polynomial kernel. However, there are no previously known subspace embeddings for other kernels. 
We develop efficient and provable embeddings for a large family of kernels including Gaussian kernel and other shift invariant kernels. 
These embeddings will be a key tool used by our algorithm.



\section{Overview} \label{sec:intuition}

\begin{figure*}[t]
\centering
\subfigure[\small{Compress data and compute leverage scores}]{\includegraphics[width=0.38\textwidth]{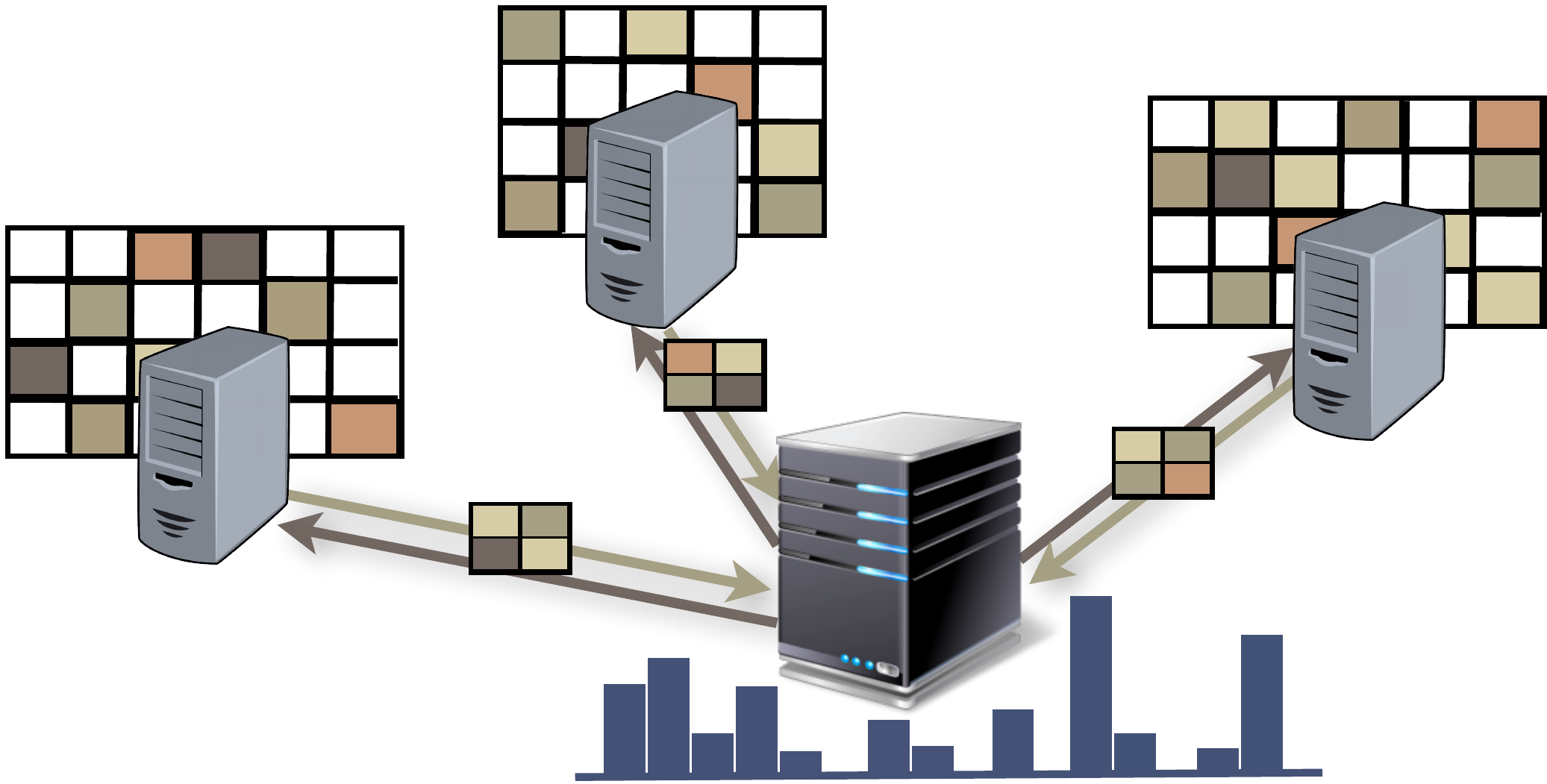}}\hspace{1.2cm}%
\subfigure[\small{Leverage score sampling}]{\includegraphics[width=0.38\textwidth]{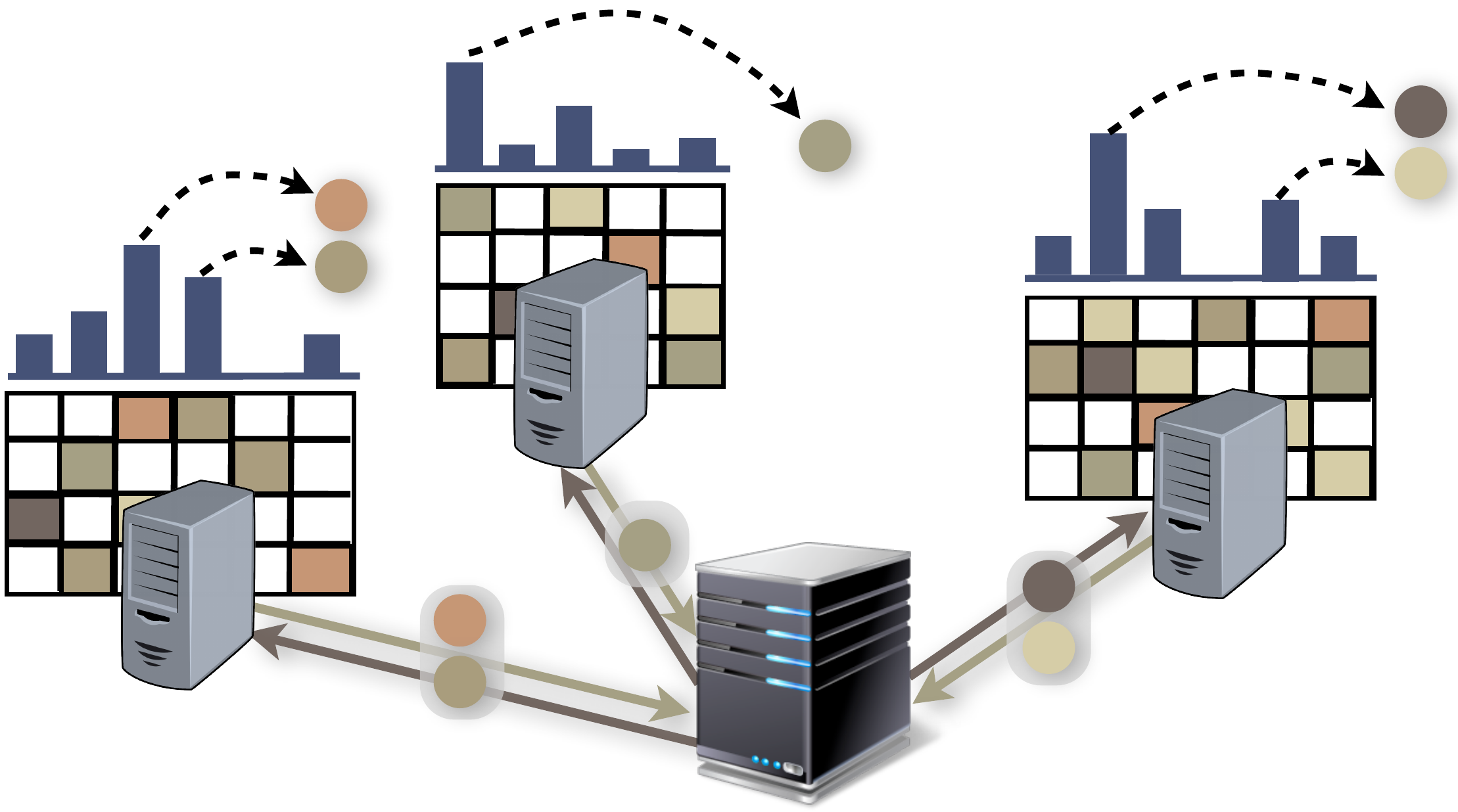}}\vspace{-4mm}
\subfigure[\small{Adaptive sampling}]{\includegraphics[width=0.38\textwidth]{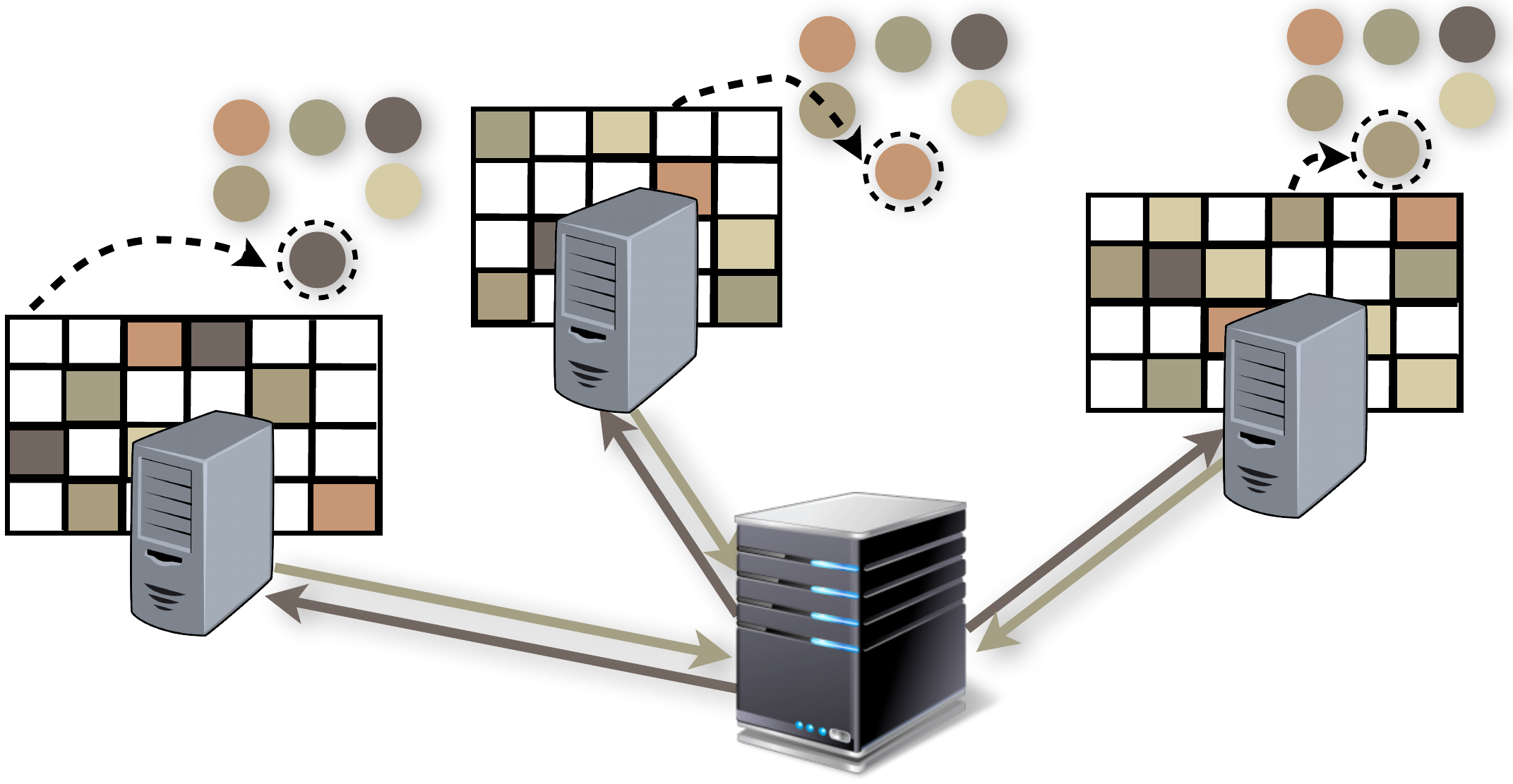}}\hspace{1cm}%
\subfigure[\small{Project data and compute KPCA}]{\includegraphics[width=0.38\textwidth]{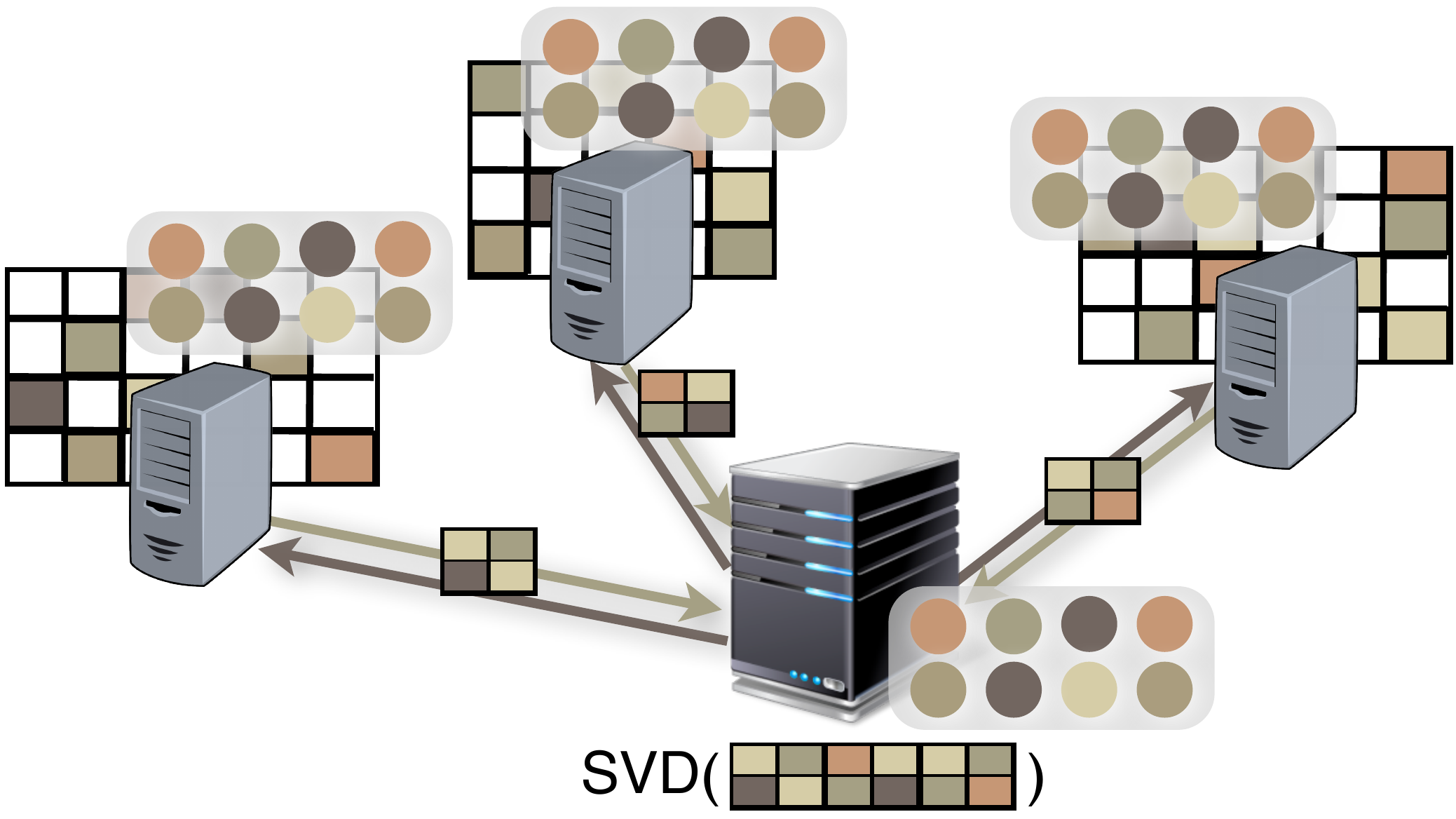}}
\vspace{-4mm}
\caption{\small{Algorithm overview. The black machine at the center is the master and the gray machines are the workers. Each worker stores its portion of the dataset, and the algorithm computes the top $k$ principle components on the whole dataset. The arrows between the machines denote the direction of communications. In each round, the communication always starts from the workers to the master (lighter arrows) and then from the master to the workers (darker arrows).
(a) Each worker compresses its data by using (kernel) subspace embeddings and sends it to the master. The master aggregates the data and computes intermediate results for leverage scores and sends back to the workers.
(b) Each worker computes the leverage scores, samples data points (denoted by circles) and then sends them to the master. The master distributes back the union of the sampled data points.
(c) Each worker conducts adaptive sampling and sends newly sampled points to the master. The master distributes back the union of all sampled points.
(d) Each worker projects its data onto the subspace spanned by the sampled data points and sends the compressed projections to the master. The master computes coefficients for the top $k$ principle components by running SVD, and then sends them back to the workers.
(best viewed in color)
}
}
\label{fig:algorithm_overview}
\end{figure*}

In view of the limitations of the related work, we instead take a different approach, which first selects a small subset of points whose span contains an approximation with relative error rate $\epsilon$, and then find a low rank approximation in their span. It is important to keep the size of the subset small and also guarantee that their span contains a good approximation (this is also called kernel column subset selection). A well known technique is to sample according to the statistical leverage scores.

{\bf Challenges.} However, this immediately raises the following technical challenges. 

{\bf I.} Computing the statistical leverage scores is prohibitively expensive.
Na\"ively computing them requires communicating all data points. There exist non-trivial fast algorithms~\cite{DriMagMahWoo12}, but they are designed for the non-distributed setting. Using them in the distributed setting leads to communication linear in the number of data points, or linear in the number of random features if one uses random features and computes the leverage scores for them. \\
Our key idea is that it is sufficient to compute the (generalized) leverage scores of the data points, i.e., the leverage scores of another matrix whose row space approximates that of the original data matrix. 
So the problem is reduced to designing kernel subspace embeddings that can approximate the row space of the data.

{\bf II.}
Even given the embedded data, it is unclear how to compute its leverage scores in a communication efficient way. Although the dimension of the embedded data is small, existing algorithms will lead to communication linear in the number of data points, which is impractical. 

{\bf III.} Simply sampling according to the generalized leverage scores does not give the desired results: a good approximation can only be obtained using a much larger rank, specifically, $O(k/\epsilon)$.

{\bf IV.} After selecting the small subset of points, we need to design a distributed algorithm to compute a good low rank approximation in their span. 

{\bf Algorithm.} We have designed a distributed kernel PCA algorithm that computes an approximated solution with relative error rate $\epsilon$ using low communication.
The algorithm operates in following key steps, each of which addresses one of the challenges mentioned above~(See Figure~\ref{fig:algorithm_overview}): 

{\bf I.} Kernel Subspace Embeddings. 
To approximate the subspace of the original data matrix, we propose subspace embeddings for a large family of kernels. 
For polynomial kernels we improve the prior work by reducing the embedding dimension and thus lowering the communication. Furthermore, we propose new subspace embeddings for kernels with random feature expansions, allowing PCA for these kernels to be computed in a communication efficient manner.  
See Section~\ref{sec:kernel_subspace} for the details.

{\bf II.} Distributed Leverage Scores.
To compute the leverage scores,  sampling with constant approximations is sufficient. We can thus drastically reduce the number of data points: first do another (non-kernel) subspace embeddings on the embedded data, and then send the result to the master for computing the scores.
See Figure~\ref{fig:algorithm_overview}(a) for an illustration and Section~\ref{sec:disLS} for the details.

{\bf III.} Sampling Representative Points.
We take a two-step approach as leverage scores alone is not good enough :
first sample according to generalized leverage scores, and then sample additional points according to their distances to the span of the points sampled in the first step.  The first step gains some coarse information about the data, and the second step use it to get the desired samples. The two steps are illustrated in Figure~\ref{fig:algorithm_overview}(b) and~\ref{fig:algorithm_overview}(c), respectively, while the details are in Section~\ref{sec:adaptive_sampling}.

{\bf IV.} Computing an Approximation.
After projecting the data to the span of the representative points, 
we sketch the projections by (non-kernel) subspace embeddings. 
We then send the compressed projections to the master and compute the solution there. 
See Figure~\ref{fig:algorithm_overview}(d) for an illustration and Section~\ref{sec:disLR} for the details.

{\bf Main Theoretical Results.} Given as input the local datasets, the rank $k$ and error parameters $\epsilon, \Delta$, our algorithm  outputs a  $(1+\epsilon, \Delta)$-approximation to the optimum with large probability. Formally, 
\begin{theorem} \label{thm:disKPCA}
Algorithm~\ref{alg:disKPCA} produces a subspace $L$ for kernel PCA on $A$ that with probability $\ge 0.99$ satisfies:
\begin{enumerate}[noitemsep,topsep=0pt,parsep=0pt,partopsep=0pt]
\item $L$ is a rank-$k$ $(1+\epsilon, 0)$-approximation when applied to polynomial kernels.
\item $L$ is a rank-$k$ $(1+\epsilon, \Delta)$-approximation when applied to shift-invariant kernels with regularization.
\end{enumerate} 
The total communication is $\tilde{O}(\frac{s\nsparse k}{\epsilon} +  \frac{s k^2}{\epsilon^3})$ words, where $\nsparse$ is the average number of nonzero entries in one data point.
\end{theorem}
The constant success probability can be boosted up to any high probability $1-\delta$ by repetition, which adds only an extra $O(\log \frac{1}{\delta})$ term to  communication and computation.

The output subspace $L$ is represented by $\Otil(k/\epsilon)$ sampled points $Y$ from $A$ (\ie, $L = \phi(Y)C$ for some coefficient matrix $C$), so $L$ can be easily communicated and the projection of any point on $L$ can be easily computed by the kernel trick.
The communication has linear dependence on the dimension and the number of workers, and has no dependence on the number of data points, which is crucial for big data scenarios.
Moreover, it does not depend on $\Delta$ (but the computation does), so the additive error can be made arbitrarily small with more computation.

The theorem also holds for other properly regularized kernels with random feature expansions (see \cite{RahRec07,DaiXieHe14} for more such kernels); the extension of our proof is straightforward.

We also make the following contributions: ({\it i}) Subspace embedding techniques for many kernels. ({\it ii}) Distributed algorithm for computing generalized leverage scores with low communication. ({\it iii}) Distributed algorithm for kernel column subset selection. 



\section{Distributed Kernel Principal Component Analysis} \label{sec:disKPCA}

Our algorithm first computes the (generalized) leverage scores that measure the non-uniform structure, then samples the desired subset of points whose span contains a good approximated solution, and finally finds such a solution in the span.

Leverage scores are critical for importance sampling in many fast randomized algorithms. 
The leverage scores are defined as follows.
\begin{definition} \label{def:leverage}
For $E \in \RR^{t \times n}$ with SVD $E = U\Sigma V^\top$, the leverage score $\ell_j$ for its $j$-th column is $\ell_j = \nbr{V_{j:}}^2.$
\end{definition}
Their importance is reflected in the following fact: suppose $E$ has rank at most $k$, and suppose $P$ is a subset of $O(\frac{k \log k}{\epsilon^2} )$ columns obtained by repeatedly sampled from the columns of $E$ according to their leverage scores, then the span of $P$ contains an $(1+\epsilon, 0)$-approximation subspace for $E$ with probability $\ge 0.99$ (see, \eg,~\cite{DriMahMut08}). Here, sampling one column according to the leverage scores $\ell_j$ means to define sampling probabilities $p_j$ such that $p_j \geq \frac{\ell_j}{4\sum_j \ell_j}$ for all $j$, and then pick one column where the $j$-th column is picked with probability $p_j$. 
Note that setting $p_j = \frac{\ell_j}{\sum_j \ell_j}$ is clearly sufficient, but a constant variance of $p_j$ is allowed at the expense of an extra constant factor in the sample size. This means that it is sufficient to compute constant approximations $\tilde{\ell}_j$ for $\ell_j$, and then sample according to $p_j = \frac{\tilde{\ell}_j}{\sum_j \tilde{\ell}_j}$.

However, even computing constant approximations of the leverage scores are non-trivial:  na\"ive approaches require SVD, which is expensive. Actually, SVD  is more expensive than the task of PCA itself.  Even ignoring computation cost, na\"ive SVD is prohibitive in the distributed setting due to its high communication cost.
Fortunately, it turns out that the leverage scores are an over kill for our purpose; it suffices to compute the generalized leverage scores, i.e., the leverage scores of a proxy matrix. 

\begin{definition} \label{def:gen_leverage}
If $E$ has rank $q$ and can approximate the row space of $M$ up to $(1+\epsilon, \Delta)$, \ie, there exists $X$ with
$$\nbr{XE - M}_F \leq (1+\epsilon) \nbr{M - [M]_k}_F + \Delta,$$
then the leverage scores of $E$ are called the generalized leverage scores of $M$ with respect to rank $q$. 
\end{definition}

This generalizes the definition in~\cite{DriMagMahWoo12} by allowing the rank of $E$ to be larger than $k$ and allowing additive error $\Delta$, which are important for our application.
The generalized leverage scores can act as the leverage scores for our purpose in the following sense.

\begin{lemma} \label{fac:leverage_sample}
Let $P$ be $O(\frac{q\log q}{\epsilon^2} )$ columns sampled from $M$ according to their generalized leverage scores w.r.t.\ rank $q$.
Then with probability $\ge 0.99$, the span of $P$ has a rank-$s$ $(1+ 2\epsilon, 2\Delta)$-approximation subspace for $M$.
\end{lemma}
\begin{proof}
It follows from combining Theorem 5 in~\cite{DriMahMut08} and the definition of the generalized leverage scores.
\end{proof}

Computing the generalized scores with respect to rank $q$ could be much more efficient, since the intrinsic dimension now becomes $q$, which can be much smaller than the ambient dimension (the number of points or the dimension of the feature space). 
However, as noted in the overview, there are still a few technical challenges.
\begin{itemize}\itemsep2pt \parskip1pt \parsep1pt
\item Efficiently find a smaller matrix $E$ that can approximate the row space of the original data. 
\item Compute the leverage scores of $E$ in a communication efficient way. 
\item The approximation solution in the span of $P$ has the same rank as $E$, which is $O(k/\epsilon)$ when we use kernel subspace embedding to obtain $E$. This is not satisfying since our final goal is to compute a rank-$k$ solution. 
\item Find a good approximation in the span of $\phi(Y)$ with low communication. 
\end{itemize}
Our final algorithm consists of four key steps, each of which addresses one of the above challenges. They are elaborated in the following four subsections respectively, and the final subsection presents the overall algorithm.


\subsection{Kernel Subspace Embeddings} \label{sec:kernel_subspace}
Recall that a subspace embedding $S$ for a matrix $M$ is such that $\nbr{SMx} \approx \nbr{Mx}$, \ie, the norm of any vector in the column space of $M$ is approximately preserved.
Subspace embeddings can also be generalized for the feature mapping of kernels, simply by setting $M =\phi(A)$, $S$ a linear mapping from $\Hcal \mapsto \RR^t$ and using the corresponding inner product.
If the data after the kernel subspace embedding is sufficient for solving the problem under consideration, then only $S\phi(A)$ in much lower dimension is needed. This is especially interesting for distributed kernel methods, since directly using the feature mapping or the kernel trick in this setting will lead to high communication cost, while the data after embedding can be much smaller and lead to much lower communication cost.

A sufficient condition for solving many problems (in particular, kernel PCA) is to preserve the low rank structure of the data. More precisely, the row space of  $S\phi(A)$ is a good approximation to that of $\phi(A)$, where the error is comparably to the best rank $k$ approximation error. Then $S\phi(A)$ can be used to compute the generalized leverage scores for $\phi(A)$, which can then be utilized to compute kernel PCA as mentioned above.


More precisely, we would like $S\phi(A)$ to approximate the row space of $\phi(A)$ up to $(1+\epsilon, \Delta)$, as required in the definition of the generalized leverage scores. We give such embeddings a particular name.

\begin{definition}\label{def:kernelOSE}
$S$ is called a $(1+\epsilon,\Delta)$-good subspace embedding for $\phi(A) \in \Hcal^n$, if there exists $X$ such that
$$
  \nbr{X (S\phi(A)) - \phi(A)}^2 \leq (1+\epsilon) \nbr{\phi(A) - [\phi(A)]_k}^2 + \Delta.
$$
\end{definition}

We now identify the sufficient conditions for  $(1+\epsilon,\Delta)$-good subspace  embeddings, which can then be used in constructing such embeddings for various kernels.

\begin{lemma}\label{lem:embedding_range}
$S$ is a $(1+\epsilon,\Delta)$-good subspace embedding for $\phi(A) \in \Hcal^n$ if it satisfies the following.
\begin{itemize}[noitemsep,nolistsep]
\item[\textbf{P1}] (Subspace Embedding): For any orthonormal $V \in \Hcal^k$ (\ie, $V^\top V$ is the identity), for all $x \in \RR^k$, $$
  \nbr{SV x} = (1\pm c)\nbr{Vx}
$$ 
where $c$ is a sufficiently small constant.
\item[\textbf{P2}] (Approximate Product): for any $M \in \Hcal^n, N \in \Hcal^k$, 
$$
  \nbr{(SN)^\top (SM) - N^\top M}_F^2 \leq {\frac \epsilon  k} \nbr{N}^2 \nbr{M}^2 + \Delta.
$$
\end{itemize}
\end{lemma}


{\bf Polynomial Kernels.} 
For polynomial kernels, there exists an efficient algorithm \TS to compute the embedding~\cite{AvrNguWoo14}. However, the embedding dimension has a quadratic dependence on the rank $k$, which will increase the communication. Fortunately, subspace embedding can be concatenated, so we can further apply another known subspace embedding such as one of those in Lemma~\ref{fac:embedding} which, though not fast for feature mapping, is fast for the already embedded data and has lower dimension. In this way, we can enjoy the benefits of both approaches. 

The guarantee of \TS in~\cite{AvrNguWoo14} and the property of the subspace embeddings in Lemma~\ref{fac:embedding} can be combined to verify \textbf{P1} and \textbf{P2}. So we have
\begin{lemma} \label{lem:poly}
For polynomial kernels $\kappa(x,y) = (\inner{x}{y})^q$, there exists an $(1+\epsilon,0)$-good subspace embedding matrix $S: \RR^{d^q} \mapsto \RR^t$ with $t = O(k/\epsilon)$. 
\end{lemma}


{\bf Kernels with Random Feature Expansions.} 
Polynomial kernels have finite dimensional feature mappings, for which the sketching seems natural. It turns out that it is possible to extend subspace embeddings to kernels with infinite dimensional feature mappings. More precisely, we propose subspace embeddings for kernels with random feature expansions, \ie, $\kappa(x,y) = \EE_\omega\sbr{ \xi_\omega(x) \xi_\omega(y) }$ for some function $\xi(\cdot)$. 
Therefore, one can approximate the kernel by using $m$ features $z_\omega(x)$ on randomly sampled $\omega$.  Such random feature expansion can be exploited for subspace embeddings: view the expansion as the ``new'' data points and apply a sketching matrix on top of it. Compared to polynomial kernels, the finite random feature expansion leads to an additional additive error term. 
Our analysis shows that bounding the additive error term only requires sufficiently large sampled size $m$, which affects the computation but does not affect the final embedding dimension and thus the communication.

In summary, the embedding is $S\phi(x) = T R(\phi(x))$, where $R(\phi(x))\in \RR^m$ is $m$ random features for $x$ and $T \in \RR^{t \times m}$ is an embedding as in Lemma~\ref{fac:embedding}. The properties \textbf{P1} and \textbf{P2} can be verified by combining Lemma~\ref{fac:embedding} and the guarantees of random features.
\begin{lemma} \label{lem:rfkernel}
For a continuous shift-invariant kernels $\kappa(x,y) = \kappa(x-y)$ with regularization, there exists an $(1+\epsilon,\Delta)$-good subspace embedding $S: \Hcal \mapsto \RR^t$ with $t = O(k/\epsilon)$.  
\end{lemma}

\subsection{Computing Leverage Scores} \label{sec:disLS}

\begin{algorithm}[!t]
\caption{Distributed Leverage Scores: $\{\tilde{\ell}^{i}_j\}=\textbf{disLS}(\cbr{E^i}_{i=1}^s, k)$} \label{alg:disLS}
  \begin{algorithmic}[1]
						\STATE Each worker $i$:  do $\frac{1}{4}$-subspace embedding $E^i T^i \in \RR^{t \times p}$ with $p=O(t)$; 
						send $E^i T^i$ to \CP. 
			\STATE \CP: QR-factorize $\sbr{E^1 T^1, \dots, E^s T^s}^\top = UZ$; 
			send $Z$ to all workers.
      \STATE Each worker $i$:  compute $\tilde{\ell}^{i}_j = \nbr{\rbr{(Z^\top)^{-1} E^i }_{:j} }^2_2$.
  \end{algorithmic}
\end{algorithm}

Given the matrix $E$ obtained from kernel subspace embedding, we would like to compute the leverage scores of $E$. First note that this cannot be done simply in a local manner: the leverage score of a column in $E^i$ is different from the leverage score of the same column in $E$. Furthermore, though data in $E$ have low dimension, communicating all points in $E$ to the master is still impractical, since it leads to communication linear in the total number of points. 

Fortunately, we only need to compute constant approximations of the scores, which allows us to use subspace embedding on $E$ to greatly reduce the number of data points. In particular, we apply a $\frac{1}{4}$-subspace embedding $T^i$ (\eg, one of those in Lemma~\ref{fac:embedding}) on each local data set $E^i$, and then send them to the master. Let $ET$ denote all the embedded data, and do QR factorization $(ET)^\top = UZ$. Now, the rows of $U^\top = \rbr{Z^\top}^{-1} ET$ are a set of basis for $ET$. Then, think of $U^\top T^\dagger = \rbr{Z^\top}^{-1} E$ as the basis for $E$, so it suffices to compute the norms of the columns in $\rbr{Z^\top}^{-1} E$.

The details are described in Algorithm~\ref{alg:disLS} and Figure~\ref{fig:algorithm_overview}(a) shows an illustration. 
The algorithm is guaranteed to output constant approximations of the leverage scores of $E$.

\begin{lemma} \label{lem:leverage}
Let $\ell^{i}_j$ be the true leverage scores of $E$. Then Algorithm~\ref{alg:disLS} outputs
$\tilde{\ell}^{i}_j = (1\pm 1/2) \ell^{i}_j$.
\end{lemma}
\begin{proof}
The algorithm can be viewed as applying an embedding $T = \diag\rbr{T^1, \dots, T^s}$ on $E$ to approximate the scores while saving the costs.  Each $T^i$ is an $\frac{1}{4}$-subspace embedding matrix, then for any $x$,
\begin{align*}
 \nbr{x^\top E T}^2 
= & \nbr{[x^\top E^1 T^1, x^\top E^2 T^2, \dots, x^\top E^s T^s]}^2 = \sum_{i=1}^s \nbr{x^\top E^i T^i}^2  = \sum_{i=1}^s (1\pm 1/4)^2 \nbr{x^\top E^i}^2 \\
= & (1\pm 1/4)^2 \nbr{x^\top E}^2.
\end{align*}
So $T$ is also $\frac{1}{4}$-subspace embedding.
Such a scheme of using embedding for approximating the scores has been analyzed (Lemma 6 in~\cite{DriMagMahWoo12}), and the lemma follows. 
\end{proof}

We note that though a constant approximation is sufficient for our purpose, but the algorithm can output $\tilde{\ell}^{i}_j = (1\pm \epsilon) \ell^{i}_j$ by doing an $\frac{\epsilon}{2}$-subspace embedding (instead of $\frac{1}{4}$), which can be useful for other applications.

\subsection{Sampling Representative Points} \label{sec:adaptive_sampling}

\begin{algorithm}[!t]
\caption{Sampling Representative Points: $Y=\textbf{RepSample}(\cbr{A^i}_{i=1}^s, \{\tilde{\ell}^{i}_j\}, k, \epsilon)$} \label{alg:adaptive}
  \begin{algorithmic}[1]
			\STATE Workers:  sample $O(k \log k)$ points according to $\{\tilde{\ell}^{i}_j \}$; 
			send to \CP;
			\STATE  \CP: send all the sampled points $P$ to the workers;\\
       \STATE Workers:  sample $O(k/\epsilon)$ points $\Ytil$ according to the square distances to $P$ in the feature space; 
			send to \CP;
			\STATE \CP: send $Y=\Ytil\cup P$ to all the workers.
  \end{algorithmic}
\end{algorithm}

Sampling directly to the leverage scores can produce a set of points $P$ such that the span of $\phi(P)$ contains 
a $(1+\epsilon, \Delta)$-approximation to $\phi(A)$. However, the rank of that approximation can be as high as $O(k/\epsilon)$, since its rank is the same as that of the embedded data (see Lemma~\ref{fac:leverage_sample}), which will be $O(k/\epsilon)$ to achieve $\epsilon$ error. 
To get a rank-$k$ approximation and also enjoy the advantage of leverage scores, we propose to combine leverage score sampling and the adaptive sampling algorithm in~\cite{DesVem2006,BouWoo14}. 


The details are presented in Algorithm~\ref{alg:adaptive}. We first sample a set $P$ of $O(k\log k)$ points according to the leverage scores, so that the span of $\phi(P)$ contains a $(2,\Delta)$-approximation. Then we use the adaptive sampling method: sample $O(k/\epsilon)$ points according to the square distances from the points to their projections on $P$ and then add them to $P$ to get the desire set $Y$ of representative points. 
Figure~\ref{fig:algorithm_overview}(b) and~\ref{fig:algorithm_overview}(c) demonstrate the two steps of the algorithm. 

Adaptive sampling has the following guarantee:

\begin{lemma} \label{lem:adaptive}
Suppose there is a $(2, \Delta)$-approximation for $\phi(A)$ in the span of $\phi(P)$.
Then with probability $\ge 0.99$, the span of $\phi(Y)$ has a rank-$k$ $(1+\epsilon, \Delta)$-approximation.
\end{lemma}

Therefore, we solves the column subset selection problem for kernels in the distributed setting, with $O(k \log k + k/\epsilon)$ selected columns and with a communication of only $O(s\rho k/\epsilon + sk^2)$. This also provides the foundation for kernel PCA task.


\subsection{Computing an Approximation} \label{sec:disLR}

\begin{algorithm}[t]
\caption{Computing an Approximation:  $L=\textbf{disLR}(\cbr{A^i}_{i=1}^s$, $Y$, $k$, $\epsilon, \Delta$)} \label{alg:disLR}
  \begin{algorithmic}[1]
      \STATE Each worker $i$: 
			compute the basis $Q$ for $\phi(Y)$ and $\Pi^i = Q^\top \phi(A^i)$; 
			do an $\epsilon$-subspace embedding $\Pi^i T^i \in \RR^{|Y| \times w}$ with $w = O(|Y|/\epsilon^2)$, and send $\Pi^i T^i$ to \CP; 
			\STATE  \CP:  concatenate $\Pi T = \sbr{\Pi^1 T^1, \dots, \Pi^s T^s}$ and send the top $k$ singular vectors $W$ of $\Pi T$ to the workers.
			\STATE Each worker $i$: set $L = QW$.
  \end{algorithmic}
\end{algorithm}

To compute a good approximation in the span of $\phi(Y)$, the na\"ive approach is to project the data to the span and compute SVD there. However, the communication will be linear in the number of data points. Subspace embedding can be used to sketch the projected data, so that the number of data points is greatly reduced.

Algorithm~\ref{alg:disKPCA} describes the details and Figure~\ref{fig:algorithm_overview}(d) shows an illustration. To compute the best rank-$k$ approximation for the projected data $\Pi$, we do a subspace embedding on the right hand side, \ie, compute $\Pi T = \sbr{\Pi^1 T^1, \dots, \Pi^s T^s}$. Then the algorithm computes the best rank-$k$ approximation $W$ for $\Pi T$, which is then a good approximation for $\Pi$ and thus $\phi(A)$. It then returns $L$, the representation of $W$ in the coordinate system of $\phi(A)$.  The output $L$ is guaranteed to be a good approximation. 

\begin{lemma} \label{lem:approxLA}
If there is a rank-$k$ $(1+\epsilon, \Delta)$-approximation subspace in the span of $\phi(Y)$, then
\[
	\nbr{LL^\top \phi(A) - \phi(A)}^2  \leq (1+\epsilon)^2\nbr{ \phi(A) - \sbr{\phi(A)}_k }^2 + (1+\epsilon)\Delta.
\]
\end{lemma}

\begin{proofsketch}
For our choice of $w$, $T^i$ is an $\epsilon$-subspace embedding matrix for $\Pi^i$. Then their concatenation $B$ is an $\epsilon$-subspace embedding for $\Pi$, the concatenation of $\Pi^i$.  
Then we can apply the idea  implicit in~\cite{KanVemWoo14}. 


By Pythagorean Theorem, the error can be factorized into 
\begin{align*}
\underbrace{\nbr{LL^\top  \phi(A) - Q^\top \phi(A)}^2}_{T1} + \underbrace{\nbr{\phi(A) - Q Q^\top \phi(A)}^2}_{T2}.
\end{align*}
Since $LL^\top = QWW^\top Q^\top$,
\begin{align*}
T1 = \nbr{WW^\top Q^\top  \phi(A) - Q^\top \phi(A)}^2.
\end{align*}

Note that $\Pi = Q^\top \phi(A)$, and $W$ is the best rank-$k$ subspace for its embedding $\Pi T$. By property of $T$ (Theorem 7 in \cite{KanVemWoo14}), it is also a good approximation for $\Pi$. So 
\begin{align*} 
T1 \approx \nbr{[Q^\top  \phi(A)]_k - Q^\top \phi(A)}^2 =\nbr{Q [Q^\top  \phi(A)]_k - QQ^\top \phi(A)}^2.
\end{align*}
Combining this with $T2$, and applying Pythagorean Theorem again, we know that the error is roughly 
\begin{align*} 
\nbr{Q[Q^\top \phi(A)]_k   - \phi(A) }^2.
\end{align*}

Now, by assumption, there is a rank-$k$ $(1+\epsilon, \Delta)$-approximation subspace $X$ in the span of $\phi(Y)$.
Since $[Q^\top \phi(A)]_k$ is the best rank-$k$ approximation  to $Q^\top \phi(A)$, 
\begin{align*} 
& \nbr{Q[Q^\top \phi(A)]_k   - \phi(A) }^2 \\
 = & \nbr{Q[Q^\top \phi(A)]_k - QQ^\top \phi(A) }^2  + \nbr{QQ^\top \phi(A) - \phi(A) }^2 \\
\le & \nbr{X - QQ^\top \phi(A) }^2  + \nbr{QQ^\top \phi(A) - \phi(A) }^2 \\
= & \nbr{X - \phi(A) }^2.
\end{align*}
The lemma then follows.
\end{proofsketch}

\subsection{Overall Algorithm}
Now, putting things together, we obtain our final algorithm for distributed kernel PCA (Algorithm~\ref{alg:disKPCA}). 
Our main result, Theorem~\ref{thm:disKPCA}, follows by combining all the lemmas in the previous subsections (with properly adjusted $\epsilon$ and $\Delta$). 

\begin{algorithm}[t]
\caption{Distributed Kernel PCA: $L=\textbf{disKPCA}(\cbr{A^i}_{i=1}^s, k, \epsilon, \Delta)$} \label{alg:disKPCA}
  \begin{algorithmic}[1]
      \STATE Each worker $i$: do a $(1/4, \Delta)$-good subspace embedding $E^i = S(\phi(A^i))  \in \RR^{t \times n_i}, t = O(k)$; 
			\STATE Compute the leverage scores: 
			$\{\tilde{\ell}^{i}_j\}=\textbf{disLS}(\cbr{E^i}_{i=1}^s, k)$;
			\STATE Sample points: $Y=\textbf{RepSample}(\cbr{A^i}_{i=1}^s, \{\tilde{\ell}^{i}_j\}, k, \epsilon)$;
			\STATE Output $L=\textbf{disLR}(\cbr{A^i}_{i=1}^s, Y, k, \epsilon, \Delta)$.
  \end{algorithmic}
\end{algorithm}


\newcommand{\figscale}{0.24}

\begin{table}
\centering
\caption{ Dataset specification: $d$ is the original feature dimension, $n$ is the number of data points, and $s$ is the total number of workers storing the dataset distributedly. Among them, bow and 20news are sparse datasets. All datasets except mnist8m are taken from UCI repository~\cite{BacLic13}~and~\cite{BalSadWhi14}.}\label{tbl:dataset}\vspace{-2mm}
\begin{tabular}{lrrr}
\toprule
Dataset & $d$ & $n$ & $s$ \\\midrule
bow & 100,000 & 8,000,000 & 200\\
higgs & 28 & 11,000,000 &  200 \\
mnist8m & 784  & 8,000,000 & 100  \\
susy & 18 & 5,000,000 & 100 \\
yearpredmsd & 90 & 463,715 & 10 \\
ctslice & 384 & 53,500 &  10  \\
20news & 61,118 & 11,269 & 5 \\
protein & 9 & 41,157 & 5 \\\midrule
har & 561 & 10,299 & 5 \\
insurance & 85 & 9,822 & 5  \\
\bottomrule
\end{tabular}
\end{table}

\section{Experiments} \label{sec:exp}

\subsection{Datasets }
We use ten datasets to evaluate our algorithm. They contain both sparse and dense data and come from a variety of different domains, such as text, images, high energy physics and biology. We use two smaller ones to benchmark against the single-machine batch KPCA algorithm while the rest are large-scale datasets with up to tens of millions of data points and hundreds of thousands dimensions. Refer to Table~\ref{tbl:dataset} for detailed specifications.

Each dataset is partitioned on different workers according to the power law distribution with exponent $2$ to simulate the distribution of the data over large networks~\cite{ClaShaNew09}. Depending on the size of the dataset, the number of workers used ranges from $5$ to $200$ (see Table~\ref{tbl:dataset} for details).

\subsection{Experiment Settings}
Since our key contribution is sampling a small set of data points intelligently, the natural alternative is uniformly sampling. We compare with two variants of uniform sampling algorithms: 1) uniformly sampling representative points and use Algorithm~\ref{alg:disLR} to get KPCA solution (denoted as uniform+disLR); 2) uniformly sampling data points and apply batch KPCA (denoted as uniform+batch KPCA).

For both algorithms, we compare the tradeoff of low rank approximation error and communication cost. Particularly, we compare the communication needed to achieve the same error. Each method is run $5$ times and the mean and the standard deviation are reported.

For polynomial kernel, the degree is  $q = 4$ and for Gaussian RBF kernel, the kernel bandwidth $\sigma$ is set to $0.2$ of the median pairwise distance among a subset of 20000 randomly chosen data points (a.k.a, the ``median trick''). For Gaussian random feature expansion, we use 2000 random features.

In all experiments, we set the number of principle components $k=10$, which is the same number for $k$-means. The algorithm specific parameters are set as follows: 1) The subspace embedding dimension for the feature expansion $t$ is 50; 2) The subspace embedding dimension for the data points $p$ is 250; 3) We vary the number of adaptively sampled points $|\tilde{Y}|$ from 50 to 400 to simulate different communication cost; 4) The subspace embedding dimension $w$ is set to equal $|Y|$.


\begin{figure*}
\centering
\subfigure[error on insurance]{\includegraphics[width=\figscale\textwidth]{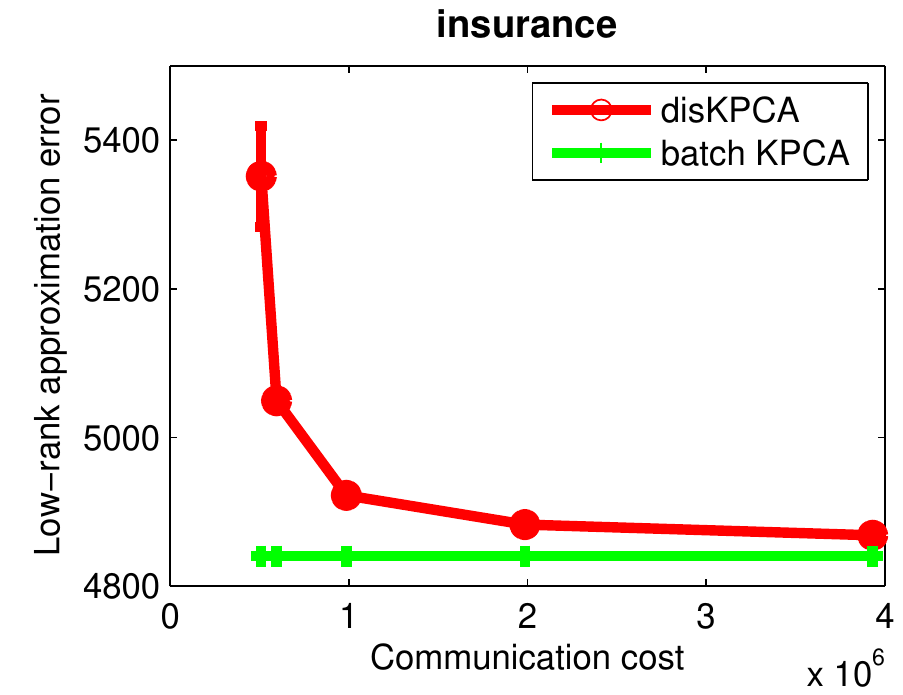}}%
\subfigure[runtime on insurance]{\includegraphics[width=\figscale\textwidth]{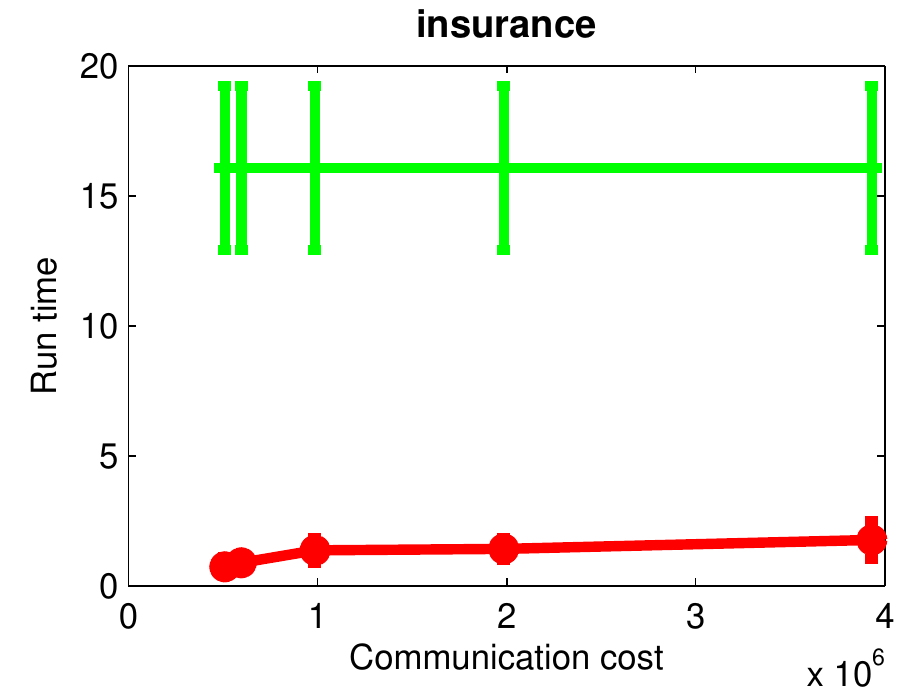}}%
\subfigure[error on har]{\includegraphics[width=\figscale\textwidth]{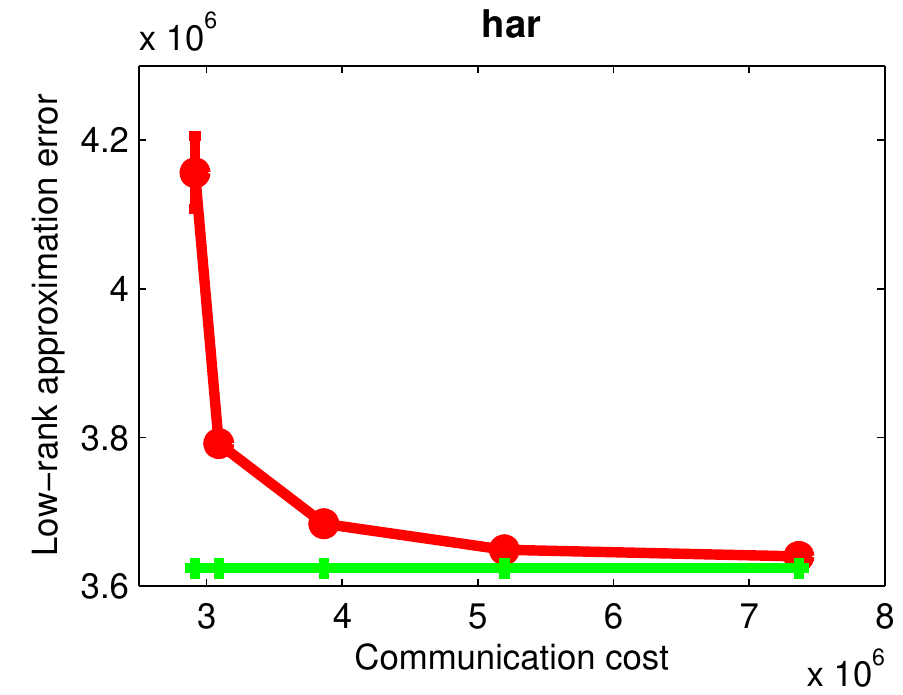}}%
\subfigure[runtime on har]{\includegraphics[width=\figscale\textwidth]{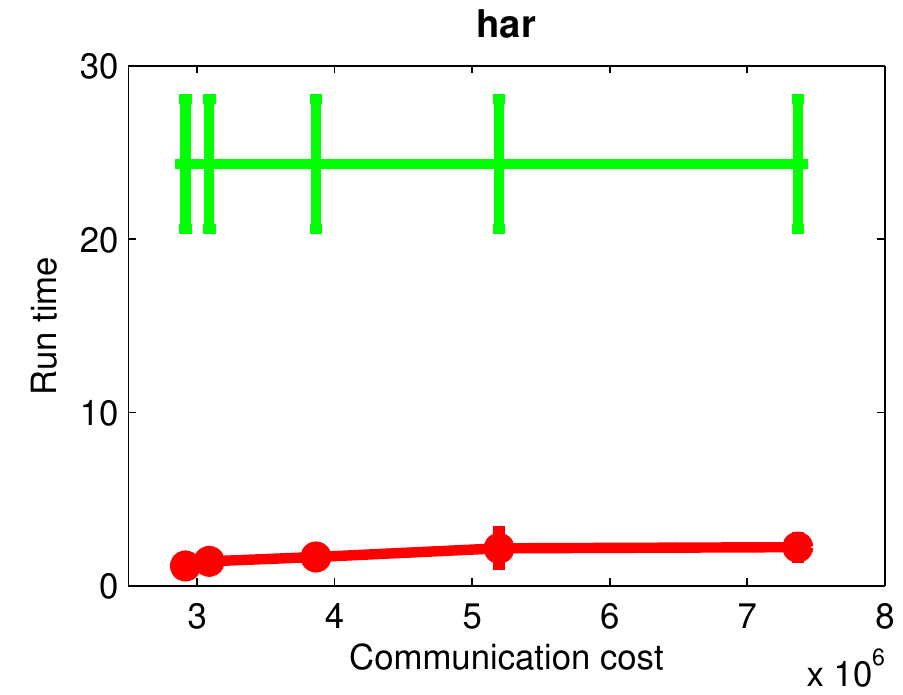}}
\vspace{-4mm}
\caption{KPCA for polynomial kernels on small datasets:   low-rank approximation error and runtime}
\label{fig:poly_kpca_err_small}
\subfigure[error on insurance]{\includegraphics[width=\figscale\textwidth]{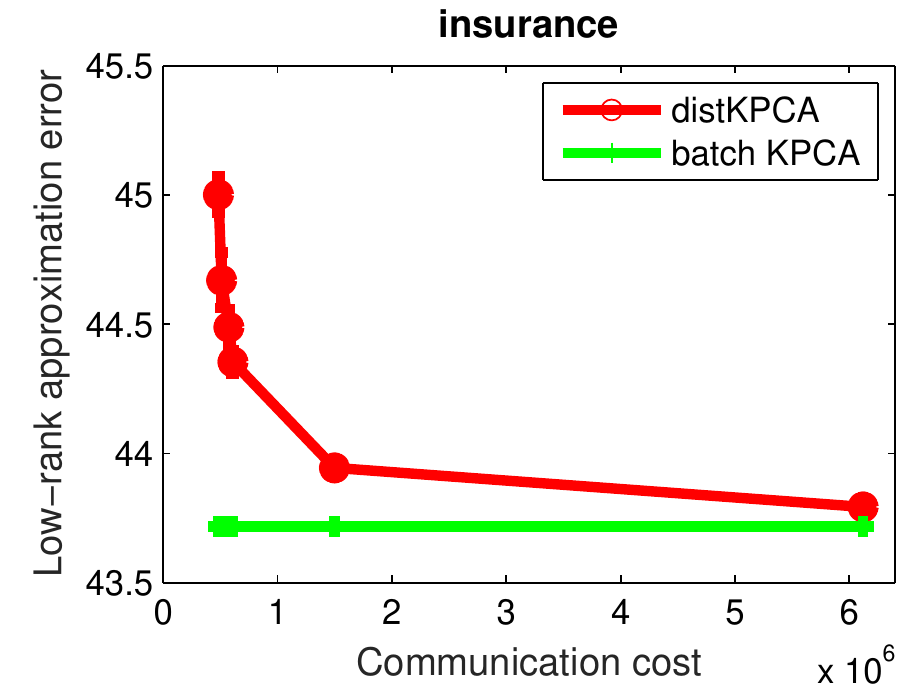}}%
\subfigure[runtime on insurance]{\includegraphics[width=\figscale\textwidth]{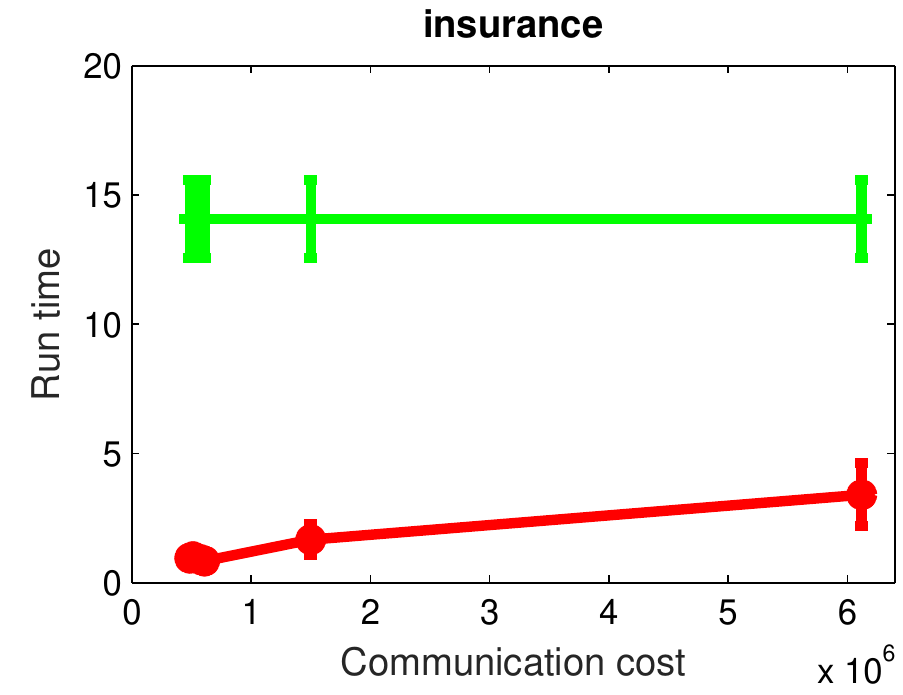}}%
\subfigure[error on har]{\includegraphics[width=\figscale\textwidth]{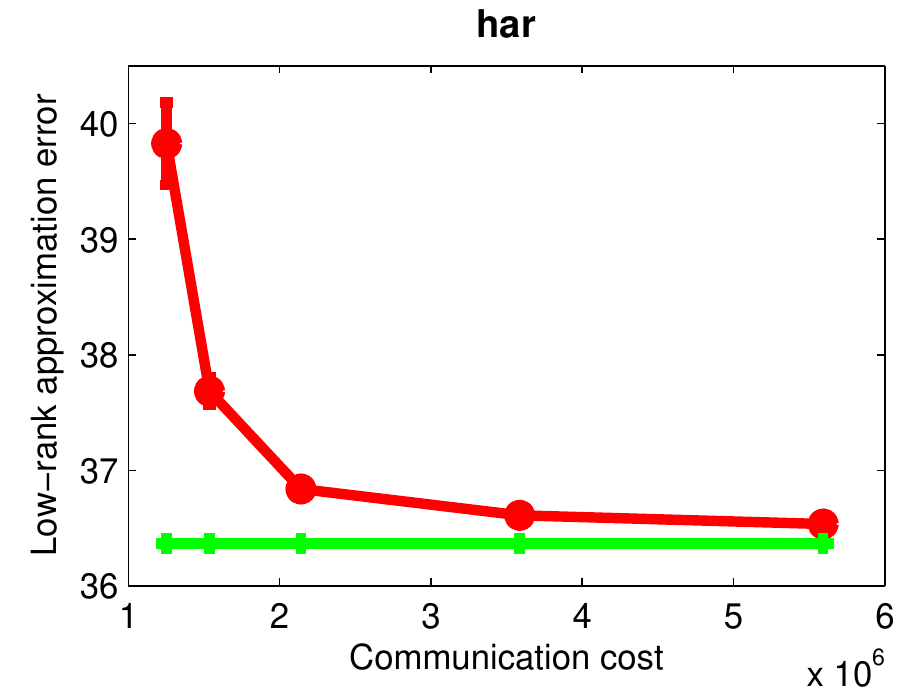}}%
\subfigure[runtime on har]{\includegraphics[width=\figscale\textwidth]{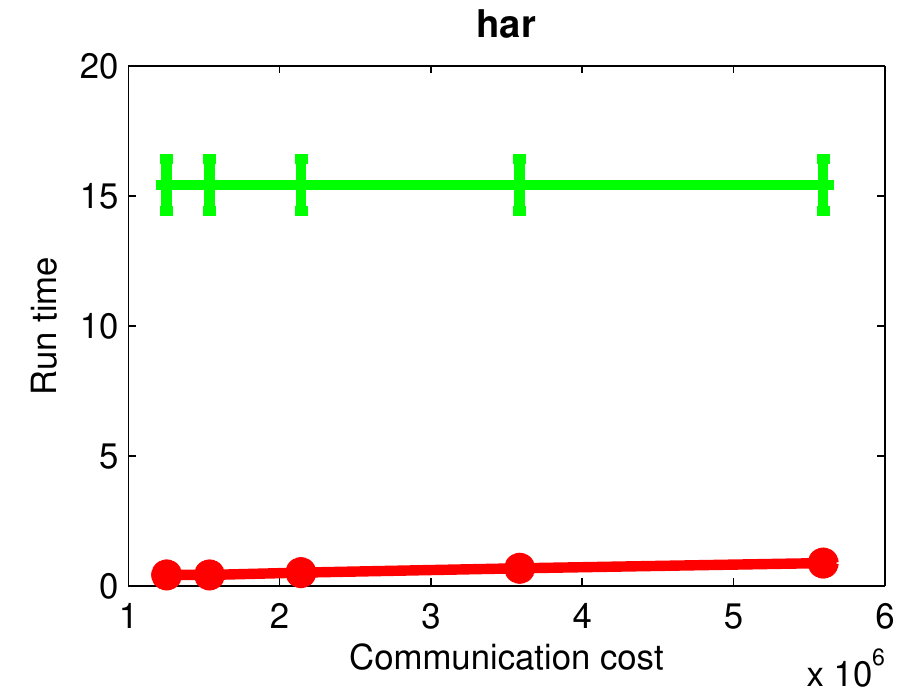}}
\vspace{-4mm}
\caption{KPCA for Gaussian kernels on small datasets: low-rank approximation error and runtime }
\label{fig:gauss_kpca_err_small}
\end{figure*}

\begin{figure*}
\centering
\includegraphics[width=\figscale\textwidth]{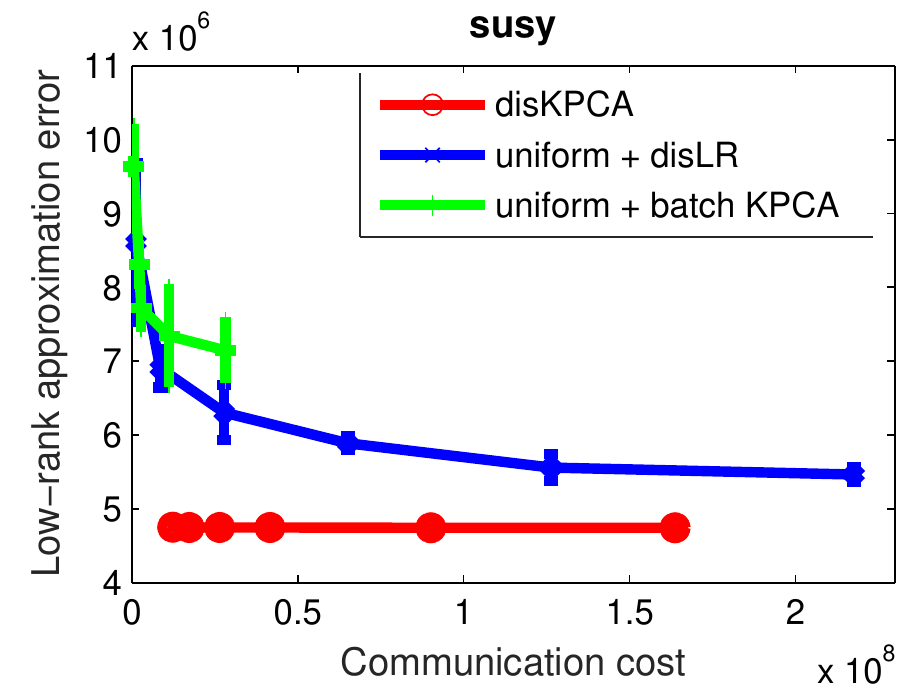}
\includegraphics[width=\figscale\textwidth]{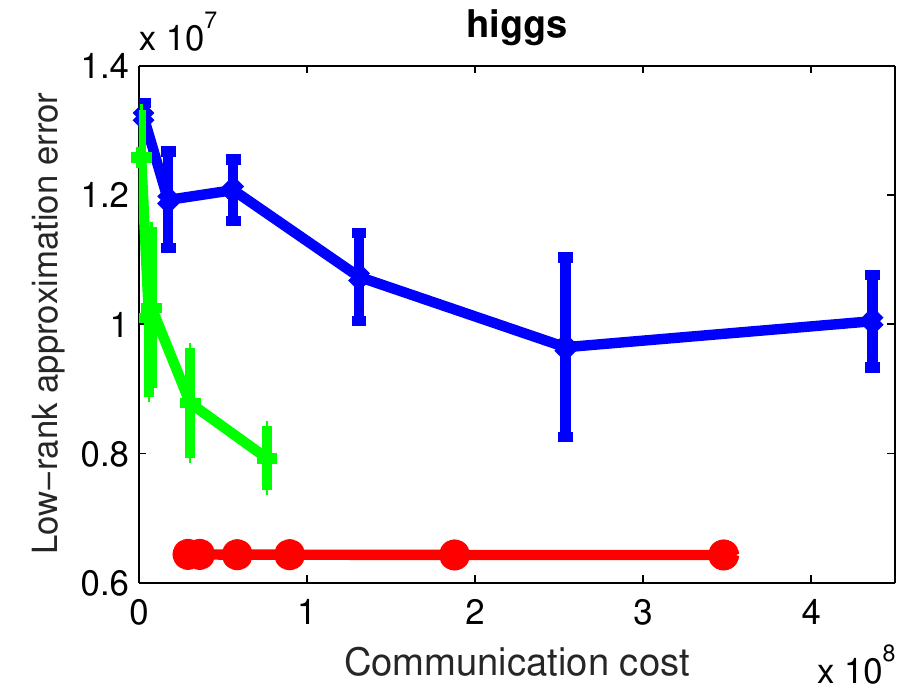}
\includegraphics[width=\figscale\textwidth]{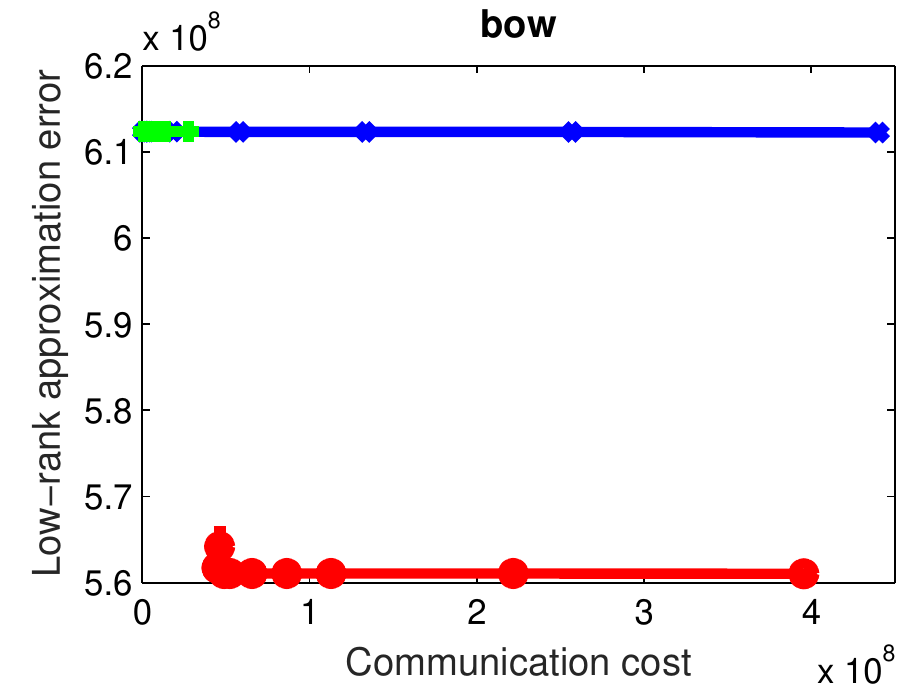}
\includegraphics[width=\figscale\textwidth]{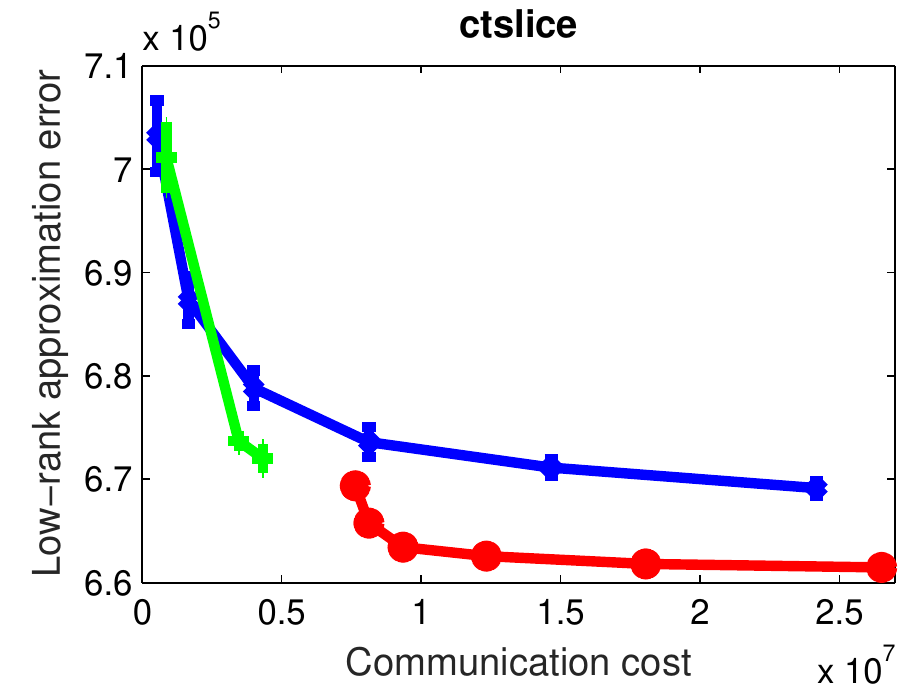}
\vspace{-4mm}
\caption{KPCA for polynomial kernels on larger datasets}
\label{fig:poly_kpca_err}
%
%
%
\includegraphics[width=\figscale\textwidth]{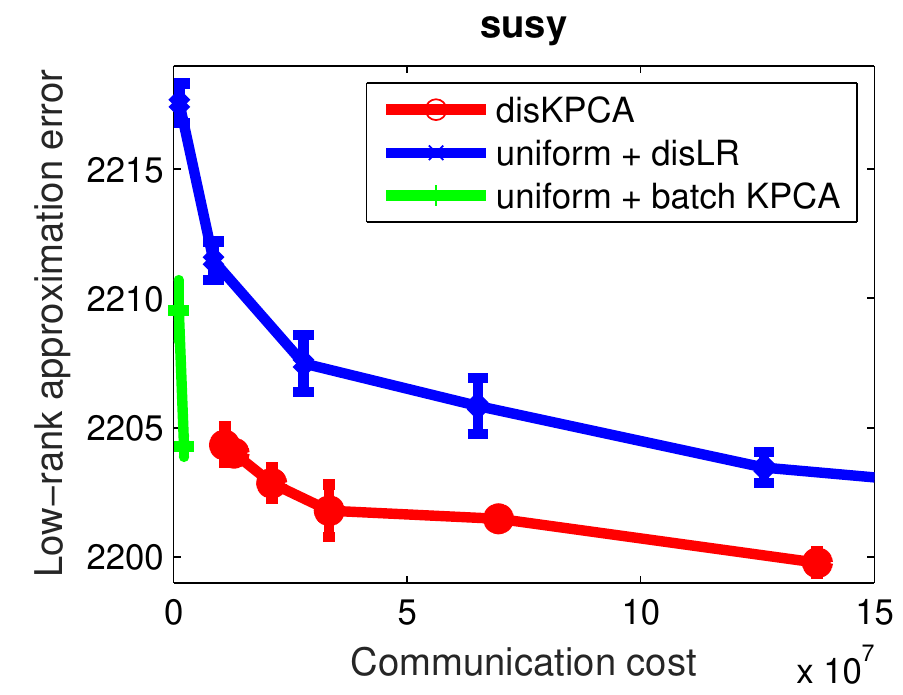}
\includegraphics[width=\figscale\textwidth]{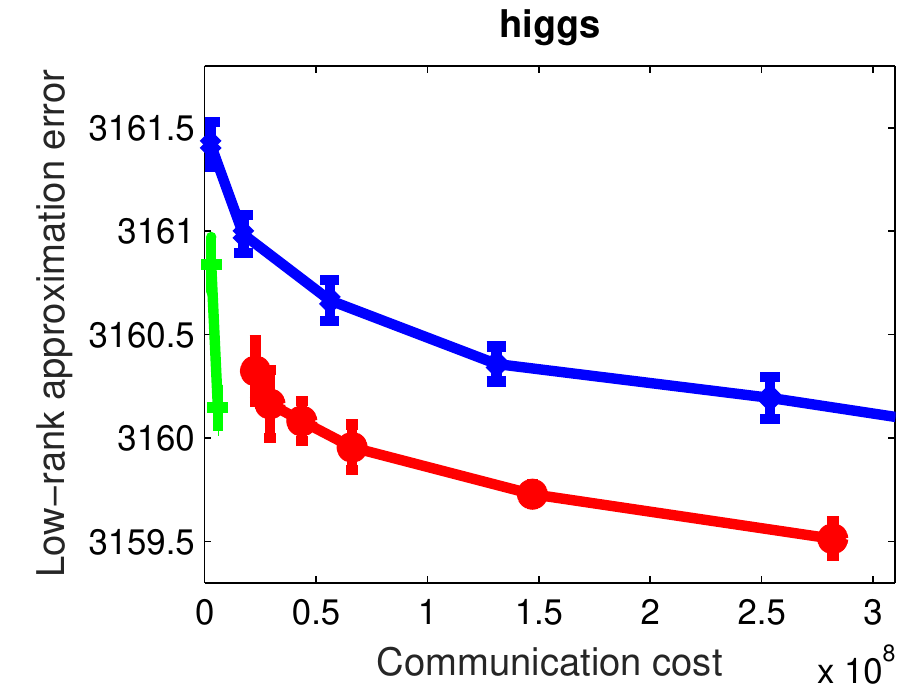}
\includegraphics[width=\figscale\textwidth]{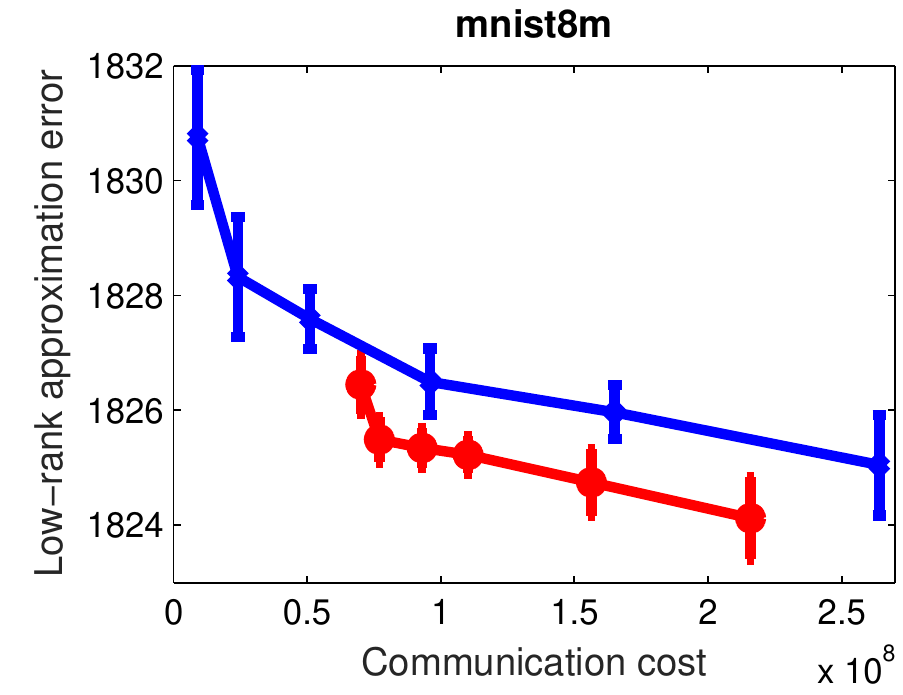}
\includegraphics[width=\figscale\textwidth]{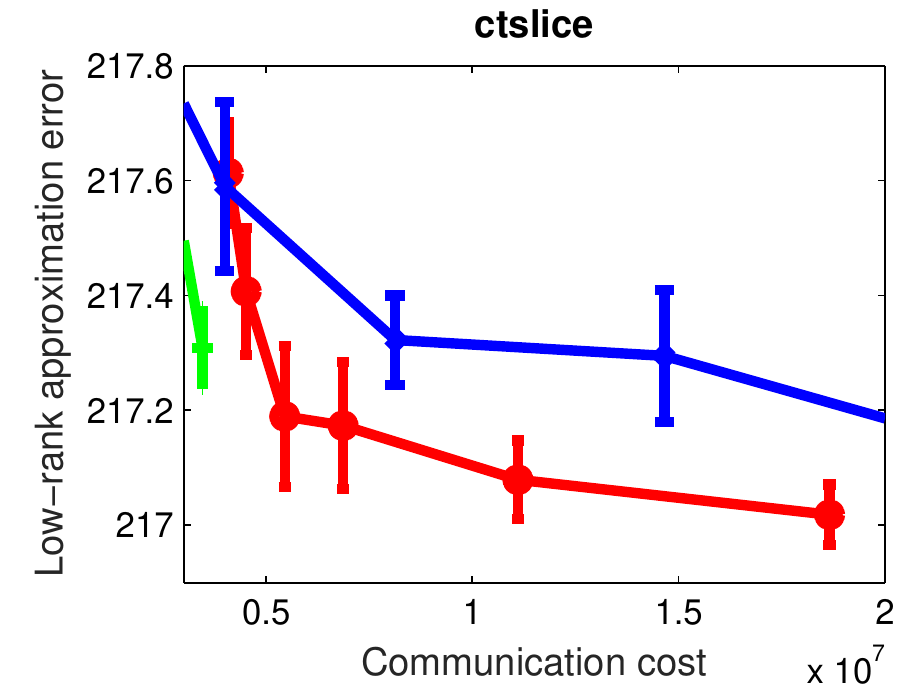}
\vspace{-4mm}
\caption{KPCA for Gaussian kernels on larger datasets}
\label{fig:gauss_kpca_err}
\end{figure*}

\subsection{Comparison with Batch Algorithm}
We compare to the ``ground-truth'' solutions produced by batch KPCA on two small datasets where it is feasible. The experiment results for the polynomial kernel and the Gaussian RBF kernel are presented in Figures~\ref{fig:poly_kpca_err_small}~and~\ref{fig:gauss_kpca_err_small}, respectively.

In both cases, the low-rank approximation error of disKPCA decreases as more communication (more represented points) is allowed. It can nearly match the optimum low-rank approximation error with much fewer data points. In addition, it is much faster: we gain a speed up of 10$\times$ by using five workers.

\subsection{Communication Efficiency}
In this set of experiments, we focus on comparing the tradeoff between communication cost and approximation accuracy on large-scale datasets. The alternative, uniform + batch KPCA, is stopped short in many experiments due to its excessive computation cost for larger number of sampled data points.

Figure~\ref{fig:poly_kpca_err}  demonstrates the performance on polynomial kernels on four large datasets. On all four datasets, our algorithm outperforms the alternatives by significant margins. Especially on bow, which is a sparse dataset, the usage of kernel embeddings takes advantage of the sparsity structure and leads to much smaller error. On other datasets, uniform + disLR cannot match the error achieved by our algorithm even when using much more communication.
 
Figure~\ref{fig:gauss_kpca_err} shows the performance on Gaussian kernels. On mnist8m, the error for uniform + batch KPCA is so large (almost twice of the errors in the figure) that it is not shown. On other datasets, disKPCA achieves significant smaller error. For example, on higgs dataset, to achieve the same approximation error, uniform + disLR requires more than $5$ times communication.
Since it does not have the communication of computing leverage scores, this means that it needs to sample much more points to get similar performance. Therefore, our algorithm is very efficient in communication cost. 

Besides polynomial and Gaussian kernels, we have also conducted experiments using arc-cos kernel~\cite{ChoSau09}. 
The arc-cosine kernels have random feature bases similar to the Rectified Linear Units (ReLU) used in deep learning. In the experiments, we use degree $n=2$ and 
Figure~\ref{fig:arccos_kpca_err} shows the results. Our algorithm consistently achieves better tradeoff between communication and approximation and the benefit is especially more pronounced on sparser dataset such as 20news.

\begin{figure}
\begin{minipage}{.5\textwidth}
\centering
		\includegraphics[width=.48\textwidth]{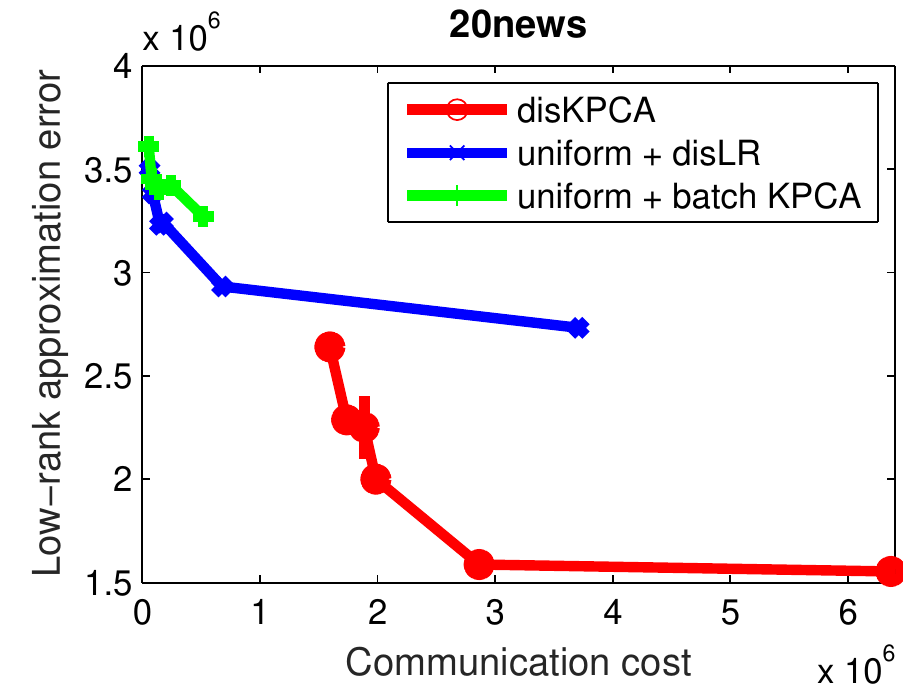}
		\includegraphics[width=.48\textwidth]{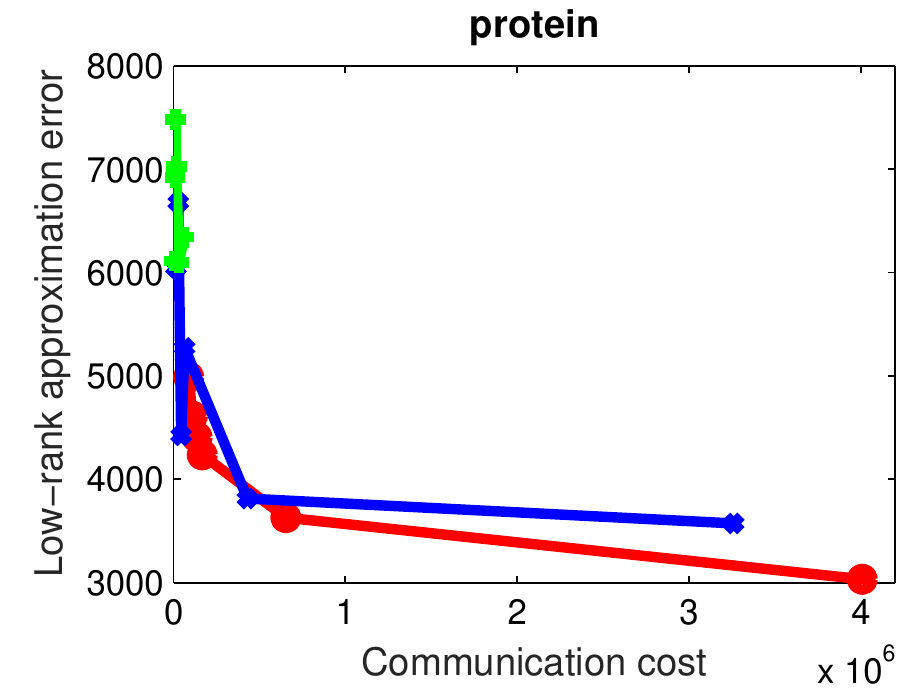}
\vspace{-4mm}
\caption{KPCA results for arc-cos kernels}
\label{fig:arccos_kpca_err}
\end{minipage}
\begin{minipage}{.5\textwidth}
\centering
		\includegraphics[width=.48\textwidth]{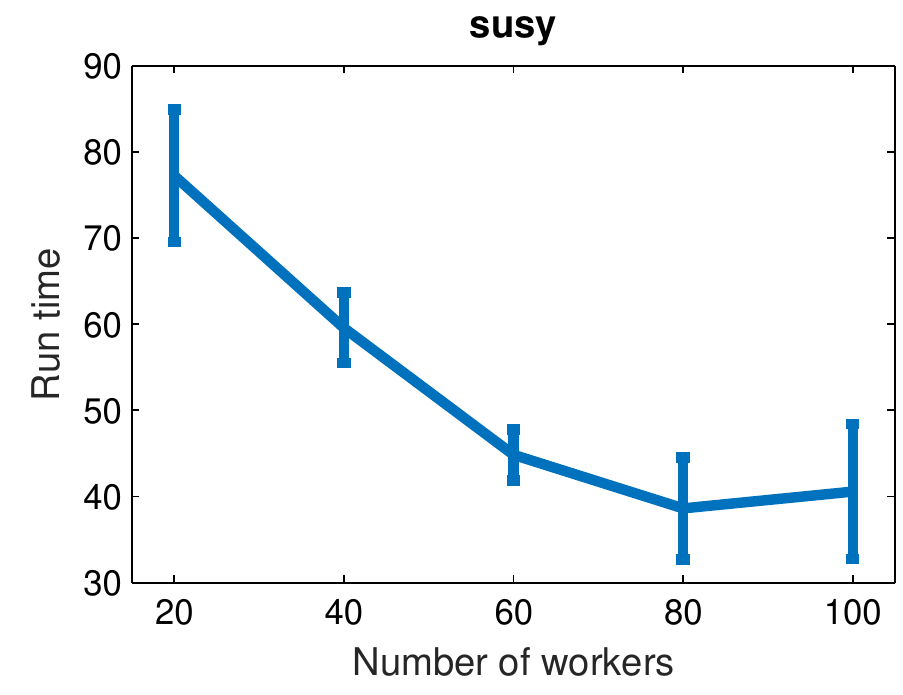}
		\includegraphics[width=.48\textwidth]{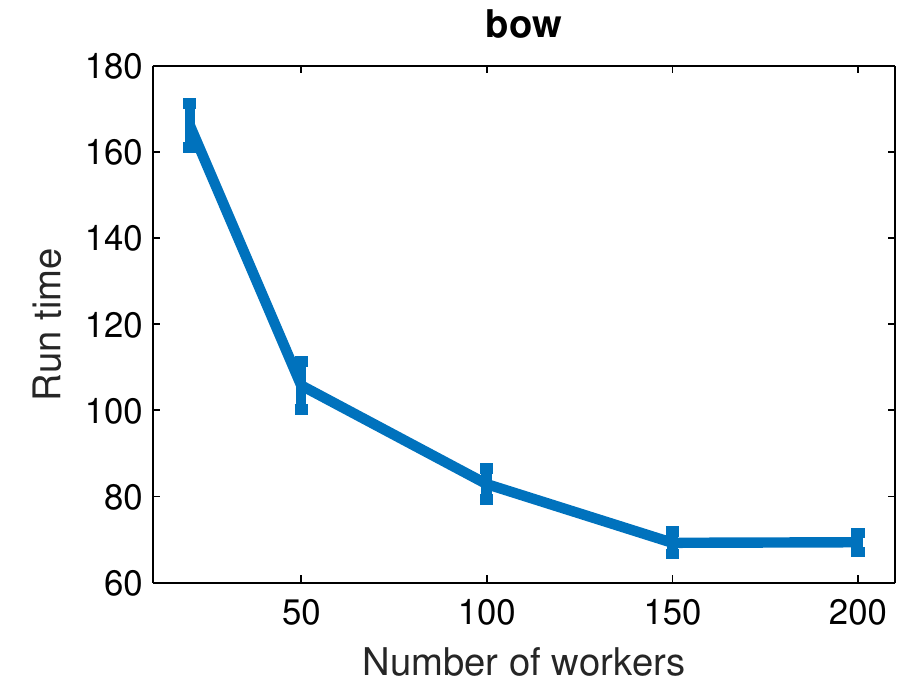}
\vspace{-4mm}
\caption{KPCA scaling results}
\label{fig:scaling}
\end{minipage}
\end{figure}

\subsection{Scaling Results}

In Figure~\ref{fig:scaling}, we present the scaling results for disKPCA. In these experiments, we vary the number of workers and record the corresponding computation time (communication time excluded). On both datasets, the runtime decreases as we use more workers, and it eventually plateaus. Our algorithm gains about $2\times$ speedup by using $4\times$ more workers. Note that our algorithm is designed to strike a good balance between communication and approximation. Even though computation complexity is not our first priority, the experiments show disKPCA still enjoys favorable scaling property.

\subsection{Distributed Spectral Clustering}
We have also experimented using $k$-means clustering as a downstream application. Such combination is known as a form of spectral clustering.
We project the data onto the top $k$ principle components and then apply a distributed $k$-means clustering algorithm~\cite{BalKanLiaWoo14}.
The evaluation criterion is the $k$-means objective, \ie, average distances to the corresponding centers, in the feature space. 

Figure~\ref{fig:poly_kmeans_err2} presents results for polynomial kernels on the 20news and susy datasets and
Figure~\ref{fig:gauss_kmeans_err2} presents results for Gaussian kernels on ctslice and yearpredmsd datasets.
Our disKPCA algorithm compares favorably with the other methods and achieves a better tradeoff of communication and error. 
This means that although the other methods require similar communication, they need to sample more data points
to achieve the same loss, demonstrating the effectiveness of our algorithm.

\begin{figure*}
\centering
\subfigure[Polynomial kernels]{
\label{fig:poly_kmeans_err2}
 \includegraphics[width=\figscale\textwidth]{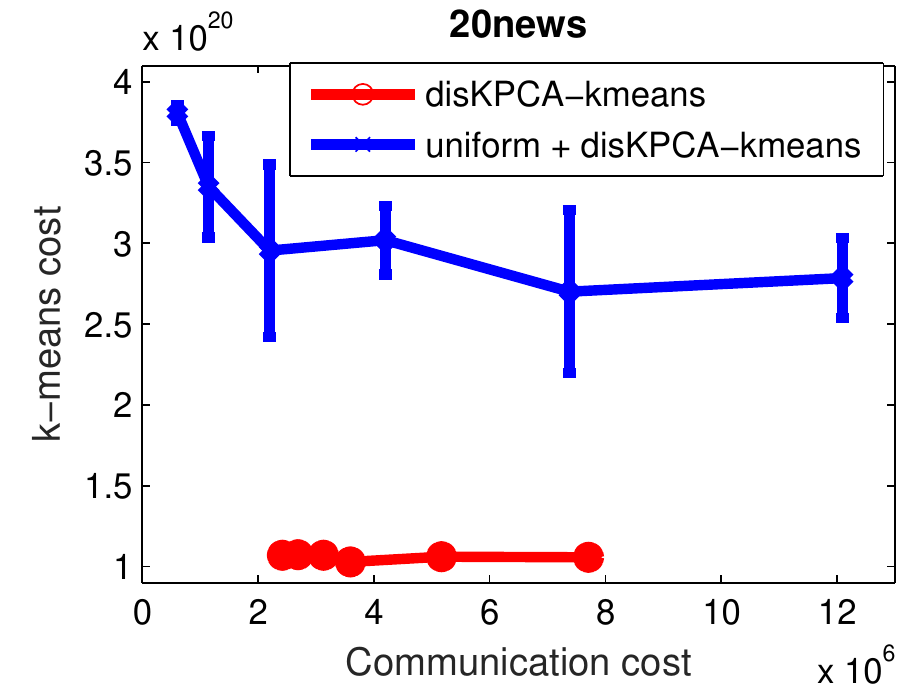}
    \includegraphics[width=\figscale\textwidth]{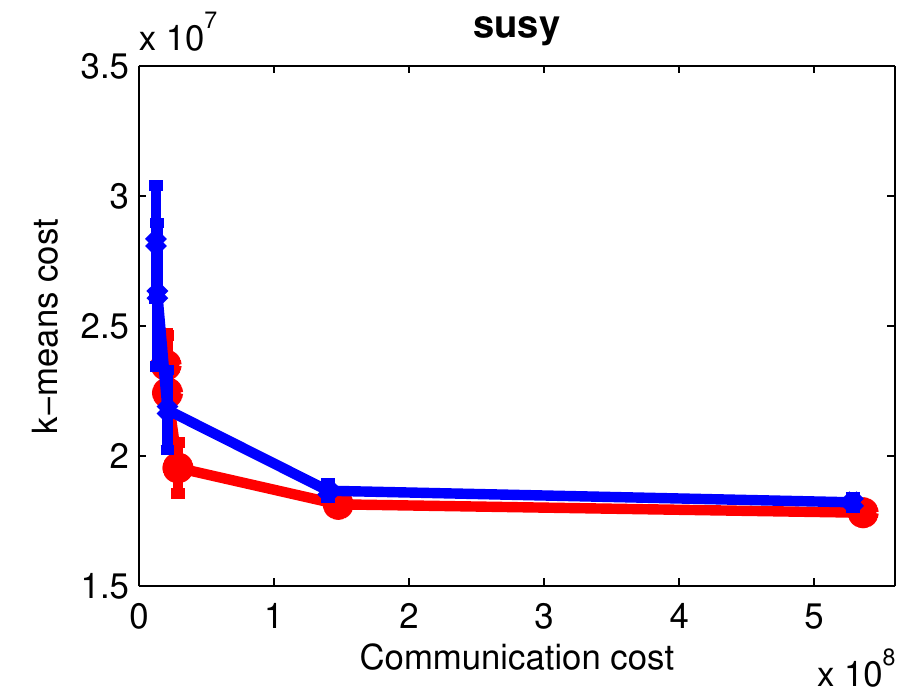}
}%
\subfigure[Gaussian kernels]{
\label{fig:gauss_kmeans_err2}
	\includegraphics[width=\figscale\textwidth]{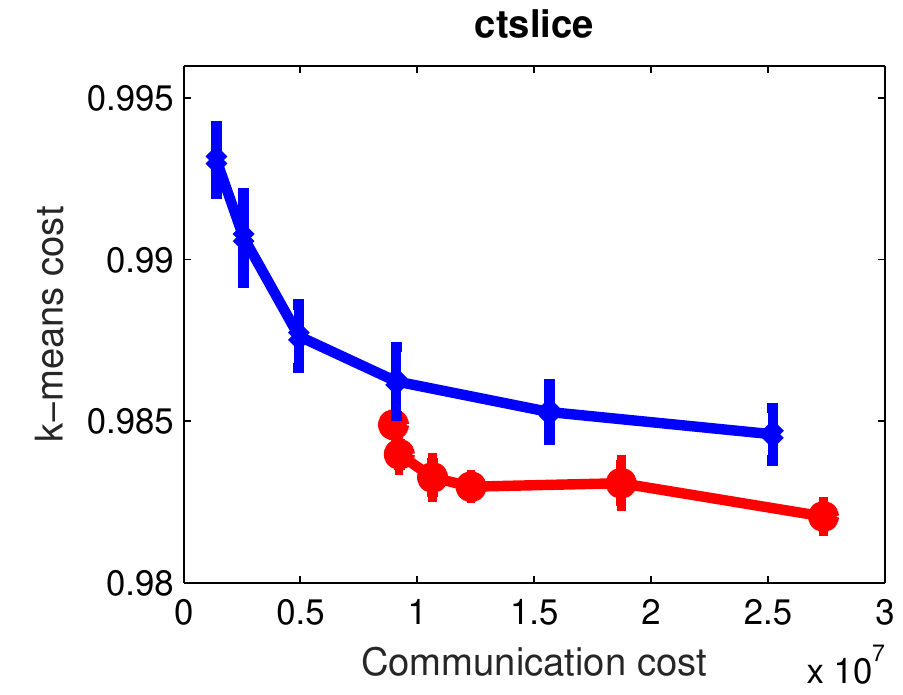}
	\includegraphics[width=\figscale\textwidth]{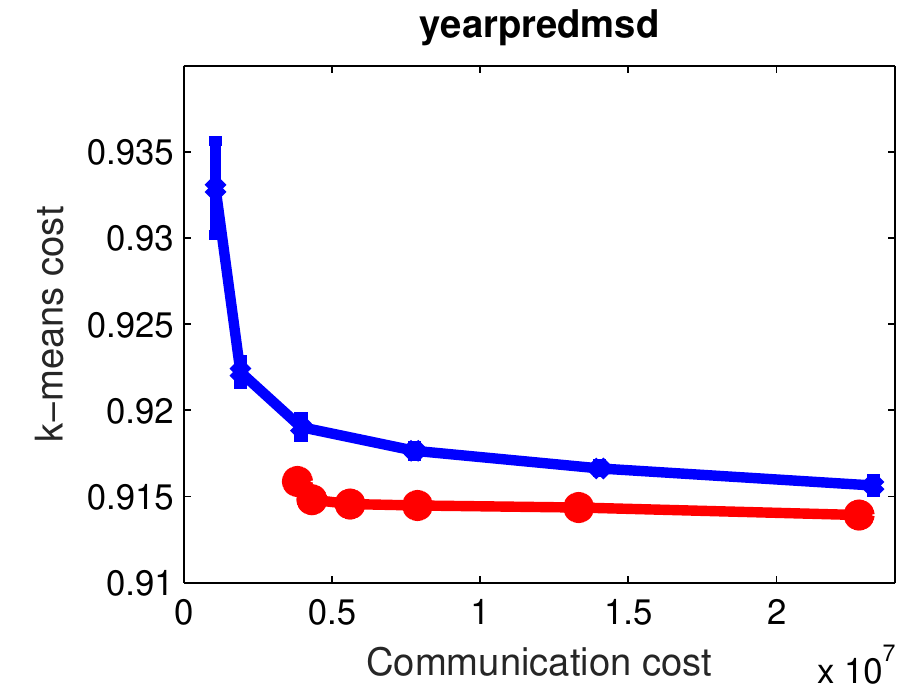}
}
\vspace{-4mm}
   \caption{KPCA + $k$-means clustering}
\label{fig:poly_kmeans}
\end{figure*}


\section{Conclusion} \label{sec:conclu}
This paper proposes a communication efficient distributed algorithm for kernel Principal Component Analysis with theoretical guarantees.  It computes a relative-error approximation compared to the best rank-$k$ subspace, using communication that nearly matches  that of the state-of-the-art algorithms for distributed linear PCA. This is the first distributed algorithm that can achieve such provable approximation and communication bounds. The experimental results show that it can achieve better performance than the baseline using the same communication budget.


\bibliographystyle{unsrt}
\bibliography{bibfile}  

\appendix

\section{Remark on Using Kernel Tricks in the Algorithms} \label{sec:implement}

In computing the final solution from $Y$ in Algorithm~\ref{alg:disKPCA}, we need to compute the projection of $\phi(A)$ onto $\phi(Y)$.
This can be done by using kernel trick and implicit Gram-Schmidt. Note that $\Pi^i = Q^\top \phi(A^i)$ where $Q$ is the basis for $\phi(P)$. Suppose $\phi(Y)$ has QR-factorization $\phi(Y) = QR $. Then $Q = \phi(Y) R^{-1}$ and $Q^\top \phi(A) = (R^{-1})^\top \phi(Y)^\top \phi(A) $ where $\phi(Y)^\top \phi(A) $ is just the kernel value between points in $Y$ and points in $A$. For $R$,  we have $Q^\top Q = (R^{-1})^\top \phi(Y)^\top \phi(Y) R^{-1} = I$, so $R^\top R = \phi(Y)^\top \phi(Y)$ and thus $R$ can be computed by factorizing the kernel matrix on $Y$.

Similarly, to compute the distance for adaptive sampling in Algorithm~\ref{alg:adaptive}, we first need to compute the projection of $\phi(A)$ onto $\phi(P)$, which can be done in the same way. Then the square distance from the original data to the projection can be computed by subtracting the square norm of the projection from the square norm of the original data.

\section{Additional Proofs} \label{sec:complete_proof}
The section provide the missing proofs in the main text.  Some proofs mainly follow known arguments in the literature, slightly generalized to handle the additive error term for our purpose; we include them here for completeness.

\subsection{Properties of Subspace Embeddings for Kernels}

\begin{oneshot}{Lemma~\ref{lem:embedding_range}}
$S$ is a $(1+\epsilon,\Delta)$-good subspace embedding for $\phi(A) \in \Hcal^n$ if it satisfies the following.
\begin{itemize}[noitemsep,nolistsep]
\item[\textbf{P1}] (Subspace Embedding): For any orthonormal $V \in \Hcal^k$ (\ie, $V^\top V$ is the identity), for all $x \in \RR^k$, $$
  \nbr{SV x} = (1\pm c)\nbr{Vx}
$$ 
where $c$ is a sufficiently small constant.
\item[\textbf{P2}] (Approximate Product): for any $M \in \Hcal^n, N \in \Hcal^k$, 
$$
  \nbr{(SN)^\top (SM) - N^\top M}_F^2 \leq {\frac \epsilon  k} \nbr{N}^2 \nbr{M}^2 + \Delta.
$$
\end{itemize}
\end{oneshot}


\begin{proof}
Let $\phi_k := [\phi(A)]_k$ and $\phi := \phi(A)$.  Suppose $\phi_k$ has SVD $\hat{U}\hat{\Sigma} \hat{V}^\top$. 
Define
\[
  \tilde{X} := \argmin_X \nbr{S(\phi_k) X - S(\phi)}_F
\]
and also note that $I = \argmin_X \nbr{\phi_k X - \phi}_F$ since $\phi_k$ is defined to be the best rank-$k$ approximation of $\phi$.

First, consider bounding $\beta := \hat{U}^\top \phi_k (\tilde{X} - I) = \hat{U}^\top  (\phi_k\tilde{X} - \phi_k)$. 
We have 
\[
  \beta  = \underbrace{S(\hat{U})^\top S(\hat{U}) \beta}_{T_1} + \underbrace{(S(\hat{U})^\top S(\hat{U}) - I) \beta}_{T_2}.
\]
For the first term, since $S$ is a linear mapping,  
\[
  T_1 = S(\hat{U})^\top S(\hat{U}) \beta = S(\hat{U})^\top S(\hat{U}\hat{U}^\top \phi_k)  (\tilde{X} - I)  = S(\hat{U})^\top S(\phi_k)  (\tilde{X} - I).
\]
Also note that by definition of $\tilde{X}$, we have $S(\phi_k)^\top (S(\phi_k) \tilde{X} - S(\phi)) = 0$. Since $\hat{U} = \phi_k W$ for some $W$, $S(\hat{U})^\top (S(\phi_k) \tilde{X} - S(\phi)) = W^\top S(\phi_k)^\top (S(\phi_k) \tilde{X} - S(\phi)) = 0$. So 
\[
  T_1  = S(\hat{U})^\top S(\phi_k)  (\tilde{X} - I) + S(\hat{U})^\top (S(\phi) - S(\phi_k) \tilde{X} ) =  S(\hat{U})^\top (S(\phi) - S(\phi_k))
\]
which leads to 
\[
  \nbr{T_1}_F^2 \leq \frac{\epsilon}{k} \nbr{\hat{U}}_\Hcal^2 \nbr{\phi-\phi_k}_\Hcal^2 + \Delta = \epsilon \nbr{\phi-\phi_k}_\Hcal^2 + \Delta. 
\]
For the second term, by the subspace embedding property, we have
\[
  \nbr{T_2}_F^2 \leq \nbr{S(\hat{U})^\top S(\hat{U}) - I}_2^2 \nbr{\beta}_F^2.
\]
Therefore, we have 
\[
  \nbr{\beta}_F^2 \leq \frac{\epsilon}{1-c^2_0} \nbr{\phi-\phi_k}_\Hcal^2 + \frac{\Delta}{1 - c_0^2}.
\]

Since $I = \argmin_X \nbr{\phi_k X - \phi}_F$, we have $\phi_k^\top (\phi_k - \phi) = 0$.  Then by Pythagorean Theorem, 
\[
  \nbr{\phi_k \tilde{X} - \phi}_\Hcal^2 =  \nbr{\phi_k - \phi}_\Hcal^2 + \nbr{\phi_k \tilde{X} - \phi_k}_\Hcal^2.
\]
Since $\nbr{\beta}_F^2 = \nbr{\phi_k \tilde{X} - \phi_k}_\Hcal^2$, we arrive at 
\[
  \nbr{\phi_k \tilde{X} - \phi}_\Hcal^2 \leq (1+O(\epsilon)) \nbr{\phi_k - \phi}_\Hcal^2 + O(\Delta).
\]

Note that $\tilde{X} = (S(\phi_k))^\dagger S(\phi)$ where $(S(\phi_k))^\dagger$ is the pseudoinverse of $S(\phi_k)$. 
Then $\phi_k \tilde{X} = \phi_k(S(\phi_k))^\dagger E$, and $W := \phi_k(S(\phi_k))^\dagger$ satisfies the statement.
\end{proof}

\subsection{Existence of Subspace Embeddings for Kernels}

\begin{oneshot}{Lemma~\ref{lem:poly}}
For polynomial kernels $\kappa(x,y) = (\inner{x}{y})^q$, there exists an $(1+\epsilon,0)$-good subspace embedding matrix $S: \RR^{d^q} \mapsto \RR^t$ with $t = O(k/\epsilon)$. 
\end{oneshot}

\begin{proof}
First use \TS~\cite{AvrNguWoo14} to bring the dimension down to $O(3^q k^2 + k/\epsilon)$. Then, we can use an i.i.d. Gaussian matrix, which reduces it to $t = O(k/\epsilon)$; or we can first use fast Hadamard transformation to bring it down to $O(k \mathrm{polylog}(k) / \epsilon)$, then multiply again by i.i.d Gaussians to bring down to $O(k/\epsilon)$.  
\textbf{P1} follows immediately from the definition, so we only need to check the matrix product. Let $S=\Omega T$ where $T$ is the \TS matrix and $\Omega$ is an i.i.d. Gaussian matrix. Since they are both subspace embedding matrices, then for any $M$ and $N$,  
\begin{align*}
\nbr{M^\top S^\top S N - M^\top T^\top T N}_F & \leq \sqrt{\frac \epsilon  k} \nbr{MT}_F \nbr{N T}_F,\\
\nbr{M^\top T^\top T N - M^\top N}_F & \leq \sqrt{\frac \epsilon  k} \nbr{M}_\Hcal \nbr{N }_\Hcal.
\end{align*}
By the subspace embedding property, $\nbr{MT}_F =  (1\pm \epsilon_0) \nbr{M}_\Hcal$ and $\nbr{NT}_F =  (1\pm \epsilon_0) \nbr{N}_\Hcal$ for some small constant $\epsilon_0$.  Combining all these bounds and choosing proper $\epsilon$, we know that $S=\Omega T$ satisfies \textbf{P2}.
\end{proof}

\begin{oneshot}{Lemma~\ref{lem:rfkernel}}
For a continuous shift-invariant kernels $\kappa(x,y) = \kappa(x-y)$ with regularization, there exists an $(1+\epsilon,\Delta)$-good subspace embedding $S: \Hcal \mapsto \RR^t$ with $t = O(k/\epsilon)$.  
\end{oneshot}

\begin{proof}
Let $R(\phi(x))$ be the random feature expansion matrix with $m$ random features.  Let $T \in \RR^{t\times m}$ be a subspace embedding matrix. We define $S(\phi(x)) := T R(\phi(x))$. Note that $S(\cdot)$ is a linear mapping from $\Hcal$ to $\RR^t$.

For the random feature expansion, we have
\begin{claim}[Claim 1 in~\cite{RahRec07}]
Let $k$ be a continuous shift-invariant positive-definite function $k(x,y) = k(x-y)$ defined on a compact set of $\RR^d$ of diameter $D$, with $k(0) = 1$ and such that $\nabla^2 k(0)$ exists. Let $\sigma_p^2$  denote the second moment of the Fourier transform of $k$.
Then $|k(x, y) - R(\phi(x))^\top R(\phi(y))| \leq \epsilon_0$ with probability $\leq 1-\delta$
when 
\[
  m = O\rbr{ \frac{d}{\epsilon_0^2} \log \rbr{ \frac{\sigma_p D}{\epsilon_0 \delta} }  }.
\]
\end{claim}

For the subspace embedding matrix, we have
\begin{claim}[Lemma 32 in~\cite{ClaWoo13}]
For $A$ and $B$ matrices with $n$ rows, 
and given  $\epsilon> 0$, there is $t = \Theta(\epsilon^{−2})$, so that for a $t \times n$ generalized sparse embedding matrix $S$, or $t \times n$ fast JL matrix, or $tlog(nd) \times n$ subsampled randomized Hadamard matrix, or leverage-score sketching
matrix for $A$ under the condition that $A$ has orthonormal columns,
\[
\Pr\sbr{ \nbr{ (SA)^\top SB - A^\top B}_F \leq \epsilon \nbr{A}_F \nbr{B}_F} \geq 1-\delta
\]
for any fixed constant $\delta > 0$.
\end{claim}

Then we have 
\begin{align*}
\nbr{ S(\phi(A))^\top S(\phi(B)) - \phi(A)^\top \phi(B)}^2_F \leq & \nbr{ S(\phi(A))^\top S(\phi(B)) - R(\phi(A))^\top R(\phi(B))}^2_F \\
&+ \nbr{ R(\phi(A))^\top R(\phi(B)) - \phi(A)^\top \phi(B)}^2_F.
\end{align*}
For the first term, we have
\begin{align*}
  \nbr{ S(\phi(A))^\top S(\phi(B)) - R(\phi(A))^\top R(\phi(B))}^2_F \leq &
	\epsilon^2 \nbr{R(\phi(A))}_F^2 \nbr{R(\phi(B))}_F^2 \\
	\leq &\epsilon^2 \nbr{\phi(A)}_\Hcal^2 \nbr{\phi(B)}_\Hcal^2  + O(\epsilon^2 ab\epsilon_0).
\end{align*}
where $a$ is the number of columns in $A$ and $b$ is the number of columns in $B$. 
Similarly the second term is bounded by $O(a\epsilon_0 + b\epsilon_0 + ab\epsilon_0^2)$. So the matrix product approximation is satisfied with
\[
\nbr{ S(\phi(A))^\top S(\phi(B)) - \phi(A)^\top \phi(B)}^2_F \leq \epsilon^2 \nbr{\phi(A)}_\Hcal^2 \nbr{\phi(B)}_\Hcal^2  + O(\epsilon^2 ab\epsilon_0) + O(a\epsilon_0 + b\epsilon_0 + ab\epsilon_0^2).
\]
 
Now consider the subspace embedding condition. Let $V$ be an orthonormal basis spanning the subspace. Then the condition is equivalent to saying $S(V)^\top S(V)$ have bounded eigenvalues in $[1\pm O(\epsilon)]$. Note that 
\[
  \nbr{S(V)^\top S(V) - V^\top V}^2_F \leq \epsilon^2 \nbr{V}_\Hcal^4 + O(\epsilon^2 k^2 \epsilon_0) + O(k\epsilon_0 + k\epsilon_0 + k^2\epsilon_0^2)
\]
where $\nbr{V}_\Hcal^4 = k^2$. 

Then choosing $t = O(k/\epsilon)$ and $m = \max\cbr{O(da^2/\Delta^2), O(db^2/\Delta^2), \Otil((\epsilon^2 ab/k^2\Delta)^2d), \Otil(dk^2/c_0^2)}$ satisfies the two conditions.
\end{proof}

\subsection{Adpative Sampling}
\begin{oneshot}{Lemma~\ref{lem:adaptive}}
Suppose there is a $(2, \Delta)$-approximation for $\phi(A)$ in the span of $\phi(P)$.
Then with probability $\ge 0.99$, the span of $\phi(Y)$ has a rank-$k$ $(1+\epsilon, \Delta)$-approximation.
\end{oneshot}

\begin{proof}
Let $Z$ denote the best rank-$k$ approximation for $\phi(A)$ in the span of $\phi(Y)$, and let $W$ denote the projection of $\phi(A)$ on the span of $\phi(P)$.
By Theorem 3 in~\cite{DesVem2006}, we have
\[
  \EE \sbr{\nbr{\phi(A) - Z}_\Hcal^2} \leq \nbr{\phi(A) - [\phi(A)]_k}_\Hcal^2 + \frac{\epsilon}{c}  \nbr{\phi(A) - W}_\Hcal^2
\]
where $\nbr{\phi(A) - W}_\Hcal^2 \leq c \nbr{\phi(A) - [\phi(A)]_k}_\Hcal^2 + \Delta$ by our assumption.
Our lemma is then proved by Markov's inequality.
\end{proof}

\subsection{Compute Approximation Subspace}

\begin{oneshot}{Lemma~\ref{lem:approxLA}}
If there is a rank-$k$ $(1+\epsilon, \Delta)$-approximation subspace in the span of $\phi(Y)$, then
\[
	\nbr{LL^\top \phi(A) - \phi(A)}^2  \leq (1+\epsilon)^2\nbr{ \phi(A) - \sbr{\phi(A)}_k }^2 + (1+\epsilon)\Delta.
\]
\end{oneshot}

\begin{proof}
For our choice of $w$, $T^i$ is an $\epsilon$-subspace embedding matrix for $\Pi^i$. Then their concatenation $B$ is an $\epsilon$-subspace embedding for $\Pi$, the concatenation of $\Pi^i$.  
Then we can apply the argument in Lemma 5 in~\cite{AvrNguWoo14} (also implicit in Theorem 1.5 in~\cite{KanVemWoo14}).  Below we give a simplified proof.

First, let $X$ denote the $(1+\epsilon, \Delta)$-approximation subspace in the span of $\phi(Y)$. Since $[Q^\top \phi(A)]_k$ is the best rank-$k$ approximation for $Q^\top \phi(A)$, 
\begin{align*}
& \nbr{Q[Q^\top \phi(A)]_k   - \phi(A) }^2 \\
 = & \nbr{Q[Q^\top \phi(A)]_k - QQ^\top \phi(A) }^2  + \nbr{QQ^\top \phi(A) - \phi(A) }^2 \\
\le & \nbr{X - QQ^\top \phi(A) }^2  + \nbr{QQ^\top \phi(A) - \phi(A) }^2 \\
= & \nbr{X - \phi(A) }^2. 
\end{align*}
Therefore,
\begin{align}
\nbr{Q[Q^\top \phi(A)]_k  - \phi(A)}_\Hcal^2  
\leq \nbr{\phi(A) - X}^2_\Hcal  
\leq  (1+\epsilon)\nbr{\phi(A) - [\phi(A)]_k}^2_\Hcal + \Delta. \label{eqn:inQ}
\end{align}

Second, we apply Theorem 7 in \cite{KanVemWoo14} on $Q^\top \phi(A)$. Note that that theorem is stated for a specific subspace embedding scheme but it holds for any subspace embedding. Then we have 
\begin{align}
 \nbr{WW^\top Q^\top \phi(A) - Q^\top \phi(A)}_\Hcal^2  
\leq  (1+\epsilon)\nbr{[Q^\top \phi(A)]_k  - Q^\top \phi(A)}_\Hcal^2. \label{eqn:inW}
\end{align}

We now bound the error using the above two claims. By Pythagorean Theorem, 
\begin{align}
 \nbr{LL^\top \phi(A) - \phi(A)}_\Hcal^2  
= \nbr{LL^\top  \phi(A) - QQ^\top \phi(A)}_\Hcal^2  +  \nbr{\phi(A) - Q Q^\top \phi(A)}_\Hcal^2. \label{eqn:error} 
\end{align}
Noting $L=QW$, the first term on the RHS is 
\begin{eqnarray*}
 \nbr{LL^\top  \phi(A) - QQ^\top \phi(A)}_\Hcal^2 
& = & \nbr{WW^\top Q^\top \phi(A) - Q^\top \phi(A)}_\Hcal^2\\
& \leq & (1+\epsilon)\nbr{[Q^\top \phi(A)]_k  - Q^\top \phi(A)}_\Hcal^2  \\
& = & (1+\epsilon) \nbr{Q[Q^\top \phi(A)]_k  - QQ^\top \phi(A)}_\Hcal^2,
\end{eqnarray*}
where the inequality is by (\ref{eqn:inW}) and the equalities are because multiplying $Q$ on the vectors in the span of $Q$ does not change their norms.
Plugging into (\ref{eqn:error}),
\begin{eqnarray*}
\nbr{LL^\top \phi(A) - \phi(A)}_\Hcal^2  &\leq & (1+\epsilon) \bigg( \nbr{Q[Q^\top \phi(A)]_k  - QQ^\top \phi(A)}_\Hcal^2 
  +  \nbr{\phi(A) - Q Q^\top \phi(A)}_\Hcal^2 \bigg) \\
&\leq &  (1+\epsilon) \nbr{Q[Q^\top \phi(A)]_k   - \phi(A) }_\Hcal^2 \\
&\leq & (1+\epsilon)^2 \nbr{\phi(A) - [\phi(A)]_k}^2_\Hcal + (1+\epsilon)\Delta,
\end{eqnarray*}
where the second inequality is by Pythagorean Theorem and the last is by (\ref{eqn:inQ}).
\end{proof}

\end{document}